\documentclass[dvipsnames]{article}

\usepackage{colm2024_conference}

\PassOptionsToPackage{table,dvipsnames}{xcolor}
\usepackage[table,dvipsnames]{xcolor}
\definecolor{lightgray}{gray}{0.9}
\definecolor{quotecolor}{RGB}{70,70,70}
\definecolor{lightpurple}{RGB}{230,230,250}
\usepackage{graphicx}
\usepackage{booktabs}
\usepackage{array}
\usepackage{tabularx}
\usepackage{colortbl}
\usepackage{multirow}
\usepackage{makecell}
\usepackage{adjustbox}
\usepackage{wrapfig}
\usepackage{caption}
\usepackage{subcaption}
\usepackage[table,xcdraw]{xcolor}
\definecolor{gg}{HTML}{e2f0cb}
\definecolor{db}{RGB}{155, 189, 216}
\definecolor{think_blue}{RGB}{107, 146, 207}

\usepackage{algorithm}
\usepackage{algpseudocode}

\usepackage{amsmath, amssymb, amsthm}

\usepackage{amsmath,amsfonts,bm}

\def\eqref#1{equation~\ref{#1}}

\def\1{\bm{1}}

\DeclareMathAlphabet{\mathsfit}{\encodingdefault}{\sfdefault}{m}{sl}
\SetMathAlphabet{\mathsfit}{bold}{\encodingdefault}{\sfdefault}{bx}{n}

\theoremstyle{definition}    
\newtheorem{corollary}{Corollary}
\newtheorem{proposition}[corollary]{Proposition}
\newtheorem{theorem}[corollary]{Theorem}

\newtheorem{definition}[corollary]{Definition} 

\usepackage{url}
\usepackage{xurl}
\usepackage{nicefrac}
\usepackage{fancyvrb}
\usepackage[normalem]{ulem}

\usepackage{CJKutf8}

\usepackage{tikz}
\usetikzlibrary{positioning}
\usepackage{pgfplots}
\pgfplotsset{compat=1.18}

\usepackage[raster,skins]{tcolorbox}
\tcbuselibrary{breakable}
\usepackage{transparent}

\usepackage{calligra}

\newtcolorbox{myquote}[1][]{
  enhanced,
  frame hidden,
  boxrule=0pt,
  arc=5pt,
  width=0.9\textwidth,
  before skip=5pt,
  after skip=10pt,
  boxsep=15pt,
  left=15pt, right=15pt, top=5pt, bottom=8pt,
  colback=lightpurple, opacityback=0.1,
  drop fuzzy shadow=lightpurple!60,
  sharp corners,
  #1
}

\title{From Context-Aware to Conflict-Aware: Generalizing Contrastive Decoding for Knowledge Conflict in LLMs}

\author{
 \textbf{Runze Jiang\textsuperscript{1,2}},
 \textbf{Taiqiang Wu\textsuperscript{3}},
 \textbf{Yan Wang\textsuperscript{2}},
\\
 \textbf{Bingyu Zhu\textsuperscript{2}\text{\dag}},
 \textbf{Longtao Huang\textsuperscript{2}}
\\
 \textsuperscript{1}Peking University,
 \textsuperscript{2}Alibaba Group,
 \textsuperscript{3}The University of Hong Kong
\\
 \small{
   \textbf{\textsuperscript{\dag}Corresponding authors}
 }
}

\begin{document}

\begin{CJK}{UTF8}{gbsn}
\maketitle

\begin{abstract}

When large language models generate from retrieved or augmented contexts, conflicts between external context and parametric priors remain a central reliability bottleneck.
Existing contrastive decoding methods follow a \emph{context-aware} paradigm that unilaterally amplifies context over parametric priors, overwriting correct priors when the context is erroneous.
We generalize this to the \textbf{conflict-aware} paradigm that dynamically allocates authority between prior and context based on conflict signals, rather than presupposing context trustworthiness.
We show that the affine combination of prior and context logits yields a \textbf{power family} with an inherent \textbf{regime asymmetry}: extrapolation amplifies errors unboundedly when the prior is correct, interpolation under-corrects when the context is correct, and no static regime covers both. Existing contrastive decoding methods are instances of this family, mostly extrapolative.
To evaluate both conflict directions, we propose TriState-Bench, a model-aware evaluation protocol that calibrates per-model prior knowledge to measure three conflict states: correction, resistance, and agreement.
To resolve the asymmetry, we propose Adaptive Regime Routing (ARR), which routes between regimes at each step, lifting resistance EM from below 6 to 16--33 without sacrificing correction or agreement.
Our code is available at \url{https://github.com/keith-Jiang/conflict-aware-decoding}.
\end{abstract}



\section{Introduction}

When large language models generate from retrieved or augmented contexts, conflicts between external context and parametric priors remain a central reliability bottleneck. Although LLMs encode substantial factual knowledge in their parameters \citep{petroni2019language, roberts2020much}, this parametric memory is often incomplete, outdated, or incorrect \citep{mallen2023not, kasai2023realtime}, motivating the use of retrieval-augmented generation \citep{lewis2020retrieval} and web search \citep{nakano2021webgpt} at inference time. When external context disagrees with the parametric prior, a \emph{knowledge conflict} arises \citep{xu2024knowledge}.

\begin{figure}[t]
  \centering
  \includegraphics[width=0.6\linewidth]{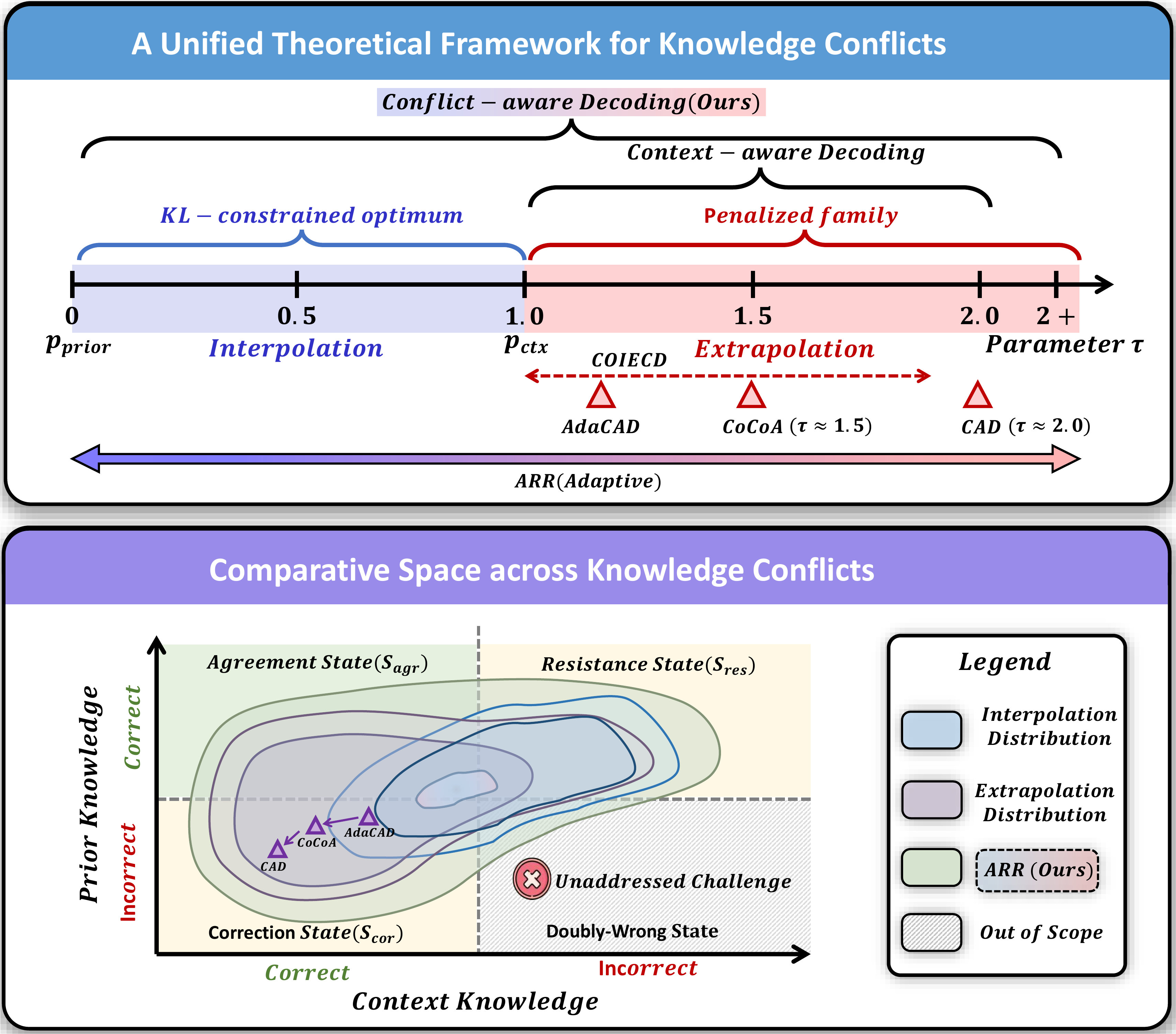}
  \caption{The unified power-family framework (upper) and comparison among proposed ARR and existing methods (lower).}
  \label{fig:framework}
\end{figure}

To address this issue, contrastive decoding methods contrast the output distributions with and without context \citep{shi2024trusting, wang2025adacad, yuan2024discerning, khandelwal2025cocoa}. These methods follow the \emph{context-aware} paradigm that implicitly assumes context is always more reliable than the prior, unidirectionally amplifying the context-over-prior logit increment. However, when the context is incorrect, this unidirectional amplification indiscriminately pushes the distribution toward it, overriding the prior's correct probability structure and biasing generation toward wrong answers.

Therefore, we generalize the problem from the context-aware paradigm to the \textbf{conflict-aware} paradigm: rather than presuming the context to be trustworthy, we dynamically allocate authority between the prior and the context at each decoding step based on conflict signals. This extends intervention to both sides and handles the two opposing conflict states, correction and resistance.

Under this paradigm, a single-scalar affine combination of $p_{\text{pri},t}$ and $p_{\text{ctx},t}$ in logit space yields the minimal parameterization: a \textbf{power family} $q_{\tau,t}(y) \propto    
p_{\text{pri},t}(y)^{1-\tau}p_{\text{ctx},t}(y)^{\tau}$, of which existing methods are all instances. The family partitions at $\tau = 1$ into two regimes: \emph{interpolation} ($\tau \in (0,1)$), the unique  optimum of a KL-constrained problem that bounds trust reallocation; and \emph{extrapolation} ($\tau > 1$), a penalized objective that suppresses prior-preferred tokens. A \textbf{regime asymmetry} emerges: extrapolation overrides a correct prior, interpolation underweights a correct context, so no single static regime handles both. Existing contrastive methods lie predominantly on the extrapolation side, structurally lacking resistance coverage.

To expose and address this asymmetry, we approach from both the evaluation and decoding sides. On the evaluation side, we propose \textbf{TriState-Bench}, a model-aware evaluation protocol that dynamically assigns each question to one of three conflict states (correction, resistance, or agreement), measuring corrective ability, prior preservation, and generation stability separately. On the decoding side, we propose \textbf{Adaptive Regime Routing (ARR)}, a theory-informed instantiation of the conflict-aware paradigm that routes between regimes per step from conflict signals in $p_{\text{prior},t}$ and $p_{\text{ctx},t}$. Across four model families, ARR covers both conflict directions, lifting resistance EM from below 6 to 16--33 without loss on correction or agreement.

Our contributions can be summarized as follows:
\begin{itemize}
    \item \textbf{Power family and regime asymmetry.} We propose a power family as the minimal parameterization of the conflict-aware paradigm, subsuming existing contrastive decoding methods, and identify an asymmetry between interpolation and extrapolation regimes.
    \item \textbf{TriState-Bench.} The first model-aware tristate benchmark for knowledge conflict, measuring correction, resistance and agreement.
    \item \textbf{Adaptive Regime Routing (ARR).} A theory-informed instantiation of the conflict-aware paradigm that dynamically routes between two regimes per step from conflict signals.
\end{itemize}
\section{Related Work}

\paragraph{Knowledge Conflict in LLMs}
Knowledge conflict arises when contextual input contradicts parametric knowledge stored in the weights \citep{petroni2019language, roberts2020much, xu2024knowledge}. Existing mitigation falls into two families: training-time fine-tuning for context faithfulness \citep{li2023large, zhou2023context} and inference-time decoding adjustments. The latter further divides into contrastive decoding, which reweights the prior distribution $p_{\text{pri},t}$ against the contextual distribution $p_{\text{ctx},t}$ \citep{shi2024trusting, wang2025adacad, yuan2024discerning, khandelwal2025cocoa}, and hidden-state intervention that modifies intermediate representations or attention patterns \citep{li2025taming, zhao2025steering}. We focus on contrastive decoding methods; hidden-state methods operate on different internal objects and fall outside our scope.

\paragraph{Contrastive Decoding under Knowledge Conflict}
Contrastive decoding \citep{li2023contrastive} was originally proposed to contrast expert and amateur models.  Subsequent work adapts it to knowledge conflict by contrasting the same model's output distributions with and without context, denoted $p_{\text{ctx},t}$ and $p_{\text{pri},t}$ respectively. Methods in this line progress from static to adaptive weighting and from single-signal to multi-signal gating: CAD \citep{shi2024trusting} amplifies the context-over-prior logit difference with a fixed $\alpha$; AdaCAD \citep{wang2025adacad} replaces the fixed weight with a Jensen--Shannon-based step-wise coefficient; COIECD \citep{yuan2024discerning} introduces token-level conflict detection; and CoCoA \citep{khandelwal2025cocoa} extends the scalar signal to a multi-signal gate.
These methods are uniformly context-aware; we generalize to a conflict-aware paradigm and subsume them as special cases of a unified power family.

\paragraph{Evaluation of Knowledge Conflict}
Existing benchmarks evaluate knowledge conflict along different axes.
NQ-Swap \citep{longpre2021entity} and NQ-Synth \citep{wang2025adacad} both reduce conflict to a single faithfulness axis: the former substitutes gold entities to test whether the model follows context over its parametric answer, while the latter replaces the context answer with the model's own output as a context-agrees-with-prior control. ClashEval \citep{wu2024clasheval} moves to a bidirectional view, separately measuring prior-biased and context-biased errors. Two gaps remain: NQ-Swap and NQ-Synth miss the correction state entirely, and ClashEval, though bidirectional, targets end-to-end LLM behavior rather than isolating decoding methods. Moreover, all three assign conflict labels statically, without conditioning on what the model actually believes. Our protocol addresses both gaps (Section~\ref{sec:benchmark_pipeline}).

\section{Preliminaries and Generalized Framework}
\subsection{Task Setup and Notation}

Given a query $x$, an external context $c$, and step $t$, we consider the same model's distributions without and with context:
\begin{equation}
\begin{aligned}
p_{\text{pri},t}(\cdot) &= p_\theta(\cdot\mid x, y_{<t}), \\
p_{\text{ctx},t}(\cdot) &= p_\theta(\cdot\mid x, c, y_{<t}).
\end{aligned}
\end{equation}
Their discrepancy reflects the influence of $c$ on the model's knowledge state.

\subsection{Background}
\label{sec:background}

\paragraph{Context-aware Decoding (CAD).} CAD applies a PMI-style adjustment that amplifies the context-over-prior increment, with a single contrastive strength $\alpha$ shared across all tokens:
\begin{equation}
    q_t^{\text{CAD}}(y)
    \propto
    p_{\text{ctx},t}(y)
    \left[
    \frac{p_{\text{ctx},t}(y)}
    {p_{\text{pri},t}(y)}
    \right]^\alpha .
\end{equation}

\paragraph{COIECD.} COIECD identifies a conflict token set $\mathcal{C}_t \subseteq V$ via the Stable Entropy Hypothesis and switches the base distribution between conflicting and non-conflicting tokens. With $g_t(y) = \log p_{\text{ctx},t}(y) - \log p_{\text{pri},t}(y)$. The score is:
\begin{equation}
    s_t^{\text{COIECD}}(y)
    =
    \begin{cases}
    \log p_{\text{pri},t}(y)+\alpha g_t(y), & y\in \mathcal C_t,\\
    \log p_{\text{ctx},t}(y)+\alpha g_t(y), & y\notin \mathcal C_t,
    \end{cases}
\end{equation}
softmax-normalization to $q_t^{\text{COIECD}}$.

\paragraph{AdaCAD.} AdaCAD replaces the static hyperparameter $\alpha$ in CAD with a dynamic weight based on Jensen-Shannon divergence, $\alpha_t^{\text{JSD}} = \operatorname{JSD}\!\left(p_{\text{pri},t} \,\|\, p_{\text{ctx},t}\right)$:
\begin{equation}
    q_t^{\text{AdaCAD}}(y)
    \propto
    p_{\text{ctx},t}(y)
    \left[
    \frac{p_{\text{ctx},t}(y)}
    {p_{\text{pri},t}(y)}
    \right]^{\alpha_t^{\text{JSD}}}.
\end{equation}

\paragraph{CoCoA.} CoCoA introduces three conflict signals (Rényi divergence, entropy gap, and contextual peakedness) and fuses them into an adaptive gating weight $\lambda_t$ to construct the final distribution:
\begin{equation}
    q_t^{\text{CoCoA}}(y)
    \propto
    p_{\text{ctx},t}(y)^{\lambda_t}
    p_{\text{pri},t}(y)^{1-\lambda_t}.
\end{equation}

\paragraph{Summary.} All four methods are \emph{context-aware}: they presume the context trustworthy, reducing the design goal to \emph{how to leverage the context more strongly}. However, a trustworthy context is only a special case. In general, the relative reliability of prior and context is not known at decoding time: sometimes the context carries correct evidence while the prior reflects outdated or incorrect memory (\emph{correction state}, $\mathcal{S}_{\text{cor}}$); sometimes the prior is more reliable while the context is noisy, incorrect, or mismatched to the query (\emph{resistance state}, $\mathcal{S}_{\text{res}}$); sometimes the two already agree, in which case excessive contrast disrupts stable generation (\emph{agreement state}, $\mathcal{S}_{\text{agr}}$). We term this broader setting the \emph{conflict-aware} paradigm. The output distribution at each decoding step is then determined by dynamically arbitrating between $p_{\text{pri},t}$ and $p_{\text{ctx},t}$ based on conflict signals $\mathcal{S}_t$:
\begin{equation}
  q_t(\cdot) = \mathcal{F}_t\!\left(p_{\text{pri},t}(\cdot),\, p_{\text{ctx},t}(\cdot),\, \mathcal{S}_t\right).
  \label{eq:conflict_aware}
\end{equation}

\subsection{A Generalized Power-Family View}
\label{sec:method_mapping}

We instantiate $\mathcal{F}t$ with the simplest parameterization: an affine combination of $\log p{\text{pri},t}$ and $\log p_{\text{ctx},t}$ in logit space, normalized to a power-family distribution:
\begin{equation}
\begin{aligned}
q_{\tau,t}(y)
=
\frac{1}{Z_{\tau,t}}\,
p_{\text{pri},t}(y)^{1-\tau}\,
p_{\text{ctx},t}(y)^\tau,\quad 
Z_{\tau,t}
=
\sum_{y'\in V}
p_{\text{pri},t}(y')^{1-\tau}\,
p_{\text{ctx},t}(y')^\tau.
\end{aligned}
\end{equation}
$\tau \in [0,1]$ interpolates between prior and context distributions; $\tau > 1$ extrapolates beyond the context by applying a negative exponent to the prior. The four methods in Section~\ref{sec:background} are all special cases that differ only in their choice of $\tau$, each occupying a fixed one-sided position (Figure~\ref{fig:framework}, Table~\ref{tab:regime_mapping}; extended discussion and derivations in Appendix~\ref{app:background}).

\begin{table}[t]
\centering
\small
\setlength{\tabcolsep}{5pt}
\renewcommand{\arraystretch}{1.2}
\begin{tabular}{@{}lccc@{}}
\toprule
\textbf{Method} & \textbf{Functional form} & \textbf{$\tau$} & \textbf{Regime} \\
\midrule
CAD     & $p_{\text{ctx},t}^{1+\alpha} p_{\text{pri},t}^{-\alpha}$ & $\tau=1+\alpha$ & Extrapolation \\
AdaCAD  & $p_{\text{ctx},t}^{1+\alpha_t^{\text{JSD}}} p_{\text{pri},t}^{-\alpha_t^{\text{JSD}}}$ & $\tau_t=1+\alpha_t^{\text{JSD}}$ & Extrapolation \\
COIECD  & $p_{\text{pri},t}^{1-\lambda_t} p_{\text{ctx},t}^{\lambda_t}$ & $\tau_t\in\{\alpha,1+\alpha\}$ & Extrapolation \\
CoCoA$^*$ & $p_{\text{pri},t}^{1-\lambda_t} p_{\text{ctx},t}^{\lambda_t}$ & $\tau_t=\lambda_t+\gamma$ & Extrapolation \\
\bottomrule
\end{tabular}
\caption{Existing contrastive decoding methods cast into the unified power family $q_{\tau,t}(y)\propto p_{\text{pri},t}(y)^{1-\tau}p_{\text{ctx},t}(y)^{\tau}$.}
\label{tab:regime_mapping}
\end{table}
\section{Regime Structure and Conflict Asymmetry}
\subsection{Interpolation vs. Extrapolation}
\begin{theorem}[Interpolation as a KL-Constrained Optimum]
\label{thm:interpolation}
For any $\epsilon \in \bigl[0,\, \mathbb{D}_{\mathrm{KL}}(p_{\text{pri},t} \,\|\, p_{\text{ctx},t})\bigr]$, consider the constrained optimization problem
\begin{equation}
\begin{aligned}
\min_{q \in \Delta(V)} \ \mathbb{D}_{\mathrm{KL}}\bigl(q \,\|\, p_{\text{pri},t}\bigr) \quad
\text{s.t.}  \ \mathbb{D}_{\mathrm{KL}}\bigl(q \,\|\, p_{\text{ctx},t}\bigr) \;\le\; \epsilon.
\end{aligned}
\end{equation}
The problem admits a unique optimum $q^\star$, given in closed form by
\begin{equation}
q^\star(y) =\frac{1}{Z_{\tau,t}}\, p_{\text{pri},t}(y)^{1-\tau}\, p_{\text{ctx},t}(y)^{\tau},
\end{equation}
 where $\tau \in [0, 1]$ is in monotone bijection with $\epsilon$. The endpoints $\tau = 0$ and $\tau = 1$ recover $p_{\text{pri},t}$ and $p_{\text{ctx},t}$, respectively.
\end{theorem}

\begin{theorem}[Extrapolation as a KL-Penalized Optimum]
\label{thm:extrapolation}
For any $\eta \in [0, 1)$, consider the penalized optimization problem
\begin{equation}
\min_{q \in \Delta(V)}\; \mathbb{D}_{\mathrm{KL}}\bigl(q \,\|\, p_{\text{ctx},t}\bigr) \;-\; \eta\, \mathbb{D}_{\mathrm{KL}}\bigl(q \,\|\, p_{\text{pri},t}\bigr).
\end{equation}
The problem admits a unique optimum $q^\star$, given in closed form by
\begin{equation}
q^\star(y) = \frac{1}{Z_{\tau,t}}\, p_{\text{pri},t}(y)^{1-\tau}\, p_{\text{ctx},t}(y)^{\tau},
\end{equation}
where $\tau = \frac{1}{1-\eta} \in [1,+\infty)$ is in monotone bijection with $\eta$. At $\eta = 0$, $\tau = 1$ recovers $p_{\text{ctx},t}$; as $\eta \uparrow 1$, $\tau \to +\infty$. The exponent $1-\tau$ is now negative: the objective actively pushes $q$ away from $p_{\text{pri},t}$ while keeping it close to $p_{\text{ctx},t}$.
\end{theorem}

The two theorems mathematically justify why the power family is the right parameterization and reveal each regime's structure (Appendix~\ref{sec:proofs}). Interpolation ($\tau \in [0,1]$) yields a \emph{bounded trust reallocation}: $q_{\tau,t}$ is the unique optimum balancing prior-closeness against context-movement, with values always sandwiched between $p_{\text{pri},t}$ and $p_{\text{ctx},t}$. Extrapolation ($\tau > 1$) advances past $p_{\text{ctx},t}$ and imposes a \emph{negative-exponent suppression on prior-favored tokens}. The two meet at $\tau = 1$, the limit of interpolation ($\epsilon \to 0$) and the start of extrapolation ($\eta \to 0$).

\subsection{Regime Asymmetry}
\label{sec:regime_asymmetry}

\begin{definition}[Pairwise Log-Odds]
\label{def:pairwise-logodds}
For any $a, b \in V$, the pairwise log-odds under the power family $q_{\tau,t}$ is defined as
\begin{equation}
\ell_{a,b}(\tau) \;:=\; \log\frac{q_{\tau,t}(a)}{q_{\tau,t}(b)} \;=\; (1-\tau)\,\ell^{\text{pri}}_{a,b} \;+\; \tau\,\ell^{\text{ctx}}_{a,b},
\end{equation}
where $\ell^{\text{pri}}_{a,b} = \log[p_{\text{pri},t}(a)/p_{\text{pri},t}(b)]$ and $\ell^{\text{ctx}}_{a,b} = \log[p_{\text{ctx},t}(a)/p_{\text{ctx},t}(b)]$.
\end{definition}

\begin{proposition}[Pairwise Reversal Threshold]
\label{prop:pairwise-reversal}
Let $\Delta_{a,b} := \ell^{\text{ctx}}_{a,b} - \ell^{\text{pri}}_{a,b}$. If $\Delta_{a,b} \neq 0$, there exists a unique \emph{pairwise reversal threshold}
\begin{equation}
\tau^\star_{a,b} := -\frac{\ell^{\text{pri}}_{a,b}}{\Delta_{a,b}}
\end{equation}
such that $\ell_{a,b}(\tau^\star_{a,b}) = 0$. As $\tau$ crosses $\tau^\star_{a,b}$, the preference of $q_{\tau,t}$ between $a$ and $b$ reverses. If $\Delta_{a,b} = 0$, the pairwise log-odds is constant and no reversal occurs.
\end{proposition}

\begin{corollary}[Conflict-State Geometry of the Crossover Point]
\label{cor:regime-state}
Let $a$ be the ground-truth token and $b$ a competing token.
\begin{itemize}
  \item \textbf{Correction ($\mathcal{S}_{\mathrm{cor}}$):} $\ell^{\mathrm{pri}}_{a,b} < 0 < \ell^{\mathrm{ctx}}_{a,b}$, so $\tau^\star_{a,b} \in (0,1)$: the flip from incorrect to correct occurs within interpolation.

\item \textbf{Resistance ($\mathcal{S}_{\mathrm{res}}$):} $\ell^{\mathrm{pri}}_{a,b} > 0 > \ell^{\mathrm{ctx}}_{a,b}$, so $\tau^\star_{a,b} \in (0,1)$ but the flip is reversed: increasing $\tau$ moves $q_\tau$ 
from correct to incorrect, and extrapolation ($\tau > 1$) further amplifies the error.

\item \textbf{Agreement ($\mathcal{S}_{\mathrm{agr}}$):} $\ell^{\mathrm{pri}}_{a,b}, \ell^{\mathrm{ctx}}_{a,b} > 0$; for $\tau \in [0,1]$, $\ell_{a,b}(\tau)$ stays positive. The family is automatically stable. 
\end{itemize}
\end{corollary}

\begin{corollary}[Length Distortion under Extrapolation]
\label{cor:generation-mode}
Let $c$ be a continuation token and $s$ a stop token. By Proposition~\ref{prop:pairwise-reversal}, $\ell_{c,s}(\tau) = \ell^{\mathrm{ctx}}_{c,s} + (\tau-1)\Delta_{c,s}$. For $\tau > 1$, the log-odds extends past   
$p_{\mathrm{ctx}}$ with unbounded displacement:
\begin{itemize}
  \item $\Delta_{c,s} > 0$: \emph{over-generation} (continuation amplified beyond context).
  \item $\Delta_{c,s} < 0$: \emph{early stopping} (stopping amplified beyond context).
  \item $\Delta_{c,s} = 0$: no overshoot.
\end{itemize}
\end{corollary}

The two corollaries reveal a two-level asymmetry. At the \emph{answer level}, interpolation confines log-odds to the convex hull of the two endpoints: it corrects the prior in $\mathcal{S}_{\mathrm{cor}}$, never overshoots either source, and is stable in $\mathcal{S}_{\mathrm{agr}}$. Extrapolation extends log-odds past $p_{\mathrm{ctx}}$ without bound, amplifying the wrong preference in $\mathcal{S}_{\mathrm{res}}$. At the \emph{generation level}, this yields over-generation ($\Delta{c,s} > 0$) or early termination ($\Delta_{c,s} < 0$). No static $\tau > 1$ is therefore universally safe: the strength that corrects in $\mathcal{S}_{\mathrm{cor}}$ catastrophically amplifies errors in $\mathcal{S}_{\mathrm{res}}$ and distorts generation length.
\section{TriState-Bench: Model-Aware Tri-State Evaluation Protocol}
\label{sec:benchmark_pipeline}

To observe the conflict states in Corollary~\ref{cor:regime-state}, we construct TriState-Bench: a model-independent fact repository $\mathcal{F}$ paired with per-model tri-state labels.

\paragraph{Step 1: Fact Repository (Model-independent).} As shown in Figure~\ref{fig:benchmark_pipeline}, each fact $f_i \in \mathcal{F}$ contains a ground-truth answer $a_i$ with alias set $\mathcal{A}_i$, an answer type drawn from three categories: person, location, and scientific term, three question variants $q_i^{(1..3)}$ targeting the same $a_i$, and a correct/incorrect context pair $c_i^+/c_i^-$ linked through an alternate entity $\tilde{a}_i$. Facts are bucket-sampled from DBpedia, grounded by Wikipedia, and rewritten by an LLM (details in Appendix~\ref{app:fact_construction}).

\paragraph{Step 2: Prior Calibration (Per-model).} The target model $M$ decodes the three questions of each fact $f_i$ without context, combining a greedy hard gate with stochastic sampling. A fact is assigned to $\mathcal{F}_{\mathrm{right}}^M$ if all three questions are matched, to $\mathcal{F}_{\mathrm{wrong}}^M$ if all are missed, or excluded as uncertain otherwise. The precise calibration rules are given in Appendix~\ref{app:prior_calibration}.

\paragraph{Step 3: Benchmark Assembly (per-model).} Given the prior labels from Step 2, we revisit the fact repository and assemble tri-state samples by pairing each fact with the corresponding pre-generated context.

\begin{itemize}
  \item \textbf{Correction $\mathcal{S}_{\mathrm{cor}}$} (prior wrong, context right): $f_i \in \mathcal{F}_{\text{wrong}}^M$ paired with $c_i^+$, targeting correction capability.
  \item \textbf{Resistance $\mathcal{S}_{\mathrm{res}}$} (prior right, context wrong): $f_i \in \mathcal{F}_{\text{right}}^M$ paired with $c_i^-$, targeting prior preservation.
  \item \textbf{Agreement $\mathcal{S}_{\mathrm{agr}}$} (prior right, context right): the same $f_i \in \mathcal{F}_{\text{right}}^M$ paired with $c_i^+$, targeting generation stability.
\end{itemize}

The doubly-wrong case (prior wrong, context wrong) is excluded as it lies outside the power family's scope. Finally, we construct 6{,}471 facts and sample 400 per benchmark under the inference budget.

\begin{figure*}[t]
\centering
  \includegraphics[width=1\linewidth]{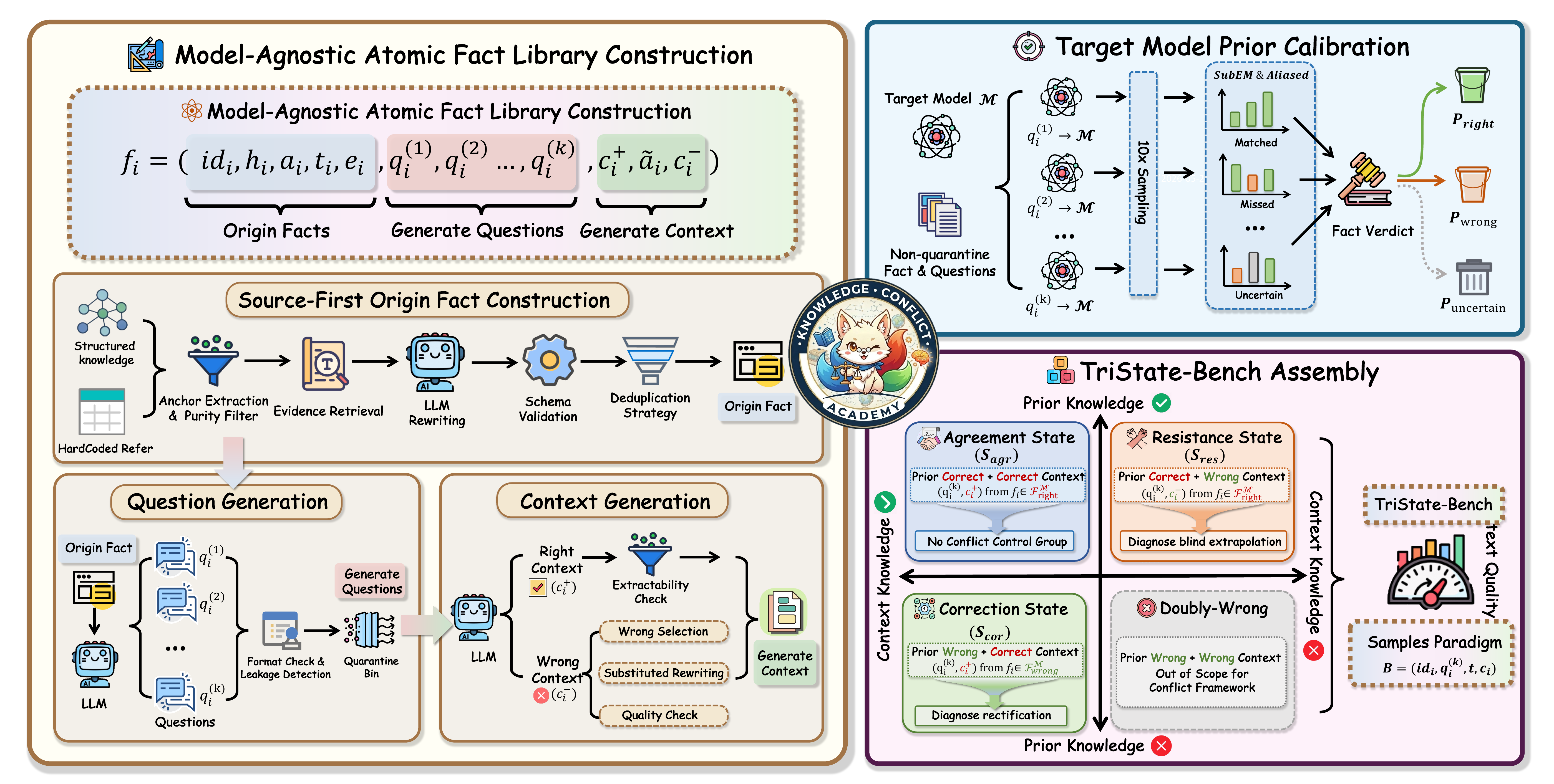}
  \caption{Overview of TriState-Bench pipeline.}
  \label{fig:benchmark_pipeline}
\end{figure*}
\section{Adaptive Regime Routing (ARR)}
\label{sec:arr}
The regime asymmetry established in Section~\ref{sec:regime_asymmetry} imposes two design requirements on $\mathcal{F}_t$ in Eq.~\ref{eq:conflict_aware}:

\begin{itemize}
    \item \textbf{Bidirectional routing with a directional gate.} Corollary~\ref{cor:regime-state} shows that interpolation is needed for $\mathcal{S}_{\mathrm{res}}$  
and $\mathcal{S}_{\mathrm{agr}}$, while extrapolation is needed for $\mathcal{S}_{\mathrm{cor}}$. The method must route $\tau$ to both sides of 1, which in turn    
requires a gate that resolves which side is trustworthy, not merely that conflict exists.
    \item \textbf{Bounded strength.} Corollary~\ref{cor:generation-mode} shows that unbounded $\tau$ causes length distortion (over-generation or early termination).   
The contrastive strength must remain bounded.
\end{itemize}

ARR instantiates the conflict-aware paradigm under these requirements with a binary gate, a bounded divergence measure, and a compositional routing rule. 

\paragraph{Gate.}   
A binary signal $d_t \in \{0, 1\}$ routes between interpolation and extrapolation. Candidate gate signals fall into two categories: \emph{divergence magnitude}
signals that detect conflict existence without resolving direction, and \emph{confidence asymmetry} signals that additionally indicate which side is more committed. Since the gate must distinguish $\mathcal{S}_{\mathrm{cor}}$ from $\mathcal{S}_{\mathrm{res}}$, directionality is essential. 
\begin{equation}
d_t = \mathbb{1}\!\left[\,
g_t > 0
\,\right],
\label{eq:arr-gate}
\end{equation}
We adopt max-probability gap as the gate in this work: 
\begin{equation}
    g_t = \max_y p_{\mathrm{ctx},t}(y) - \max_y p_{\mathrm{pri},t}(y).
\end{equation}
Intuitively, a higher top-1 probability indicates that the distribution commits more mass to a single token.
When the context side is more committed than the prior, it signals that the context has formed a concentrated prediction, characteristic of $\mathcal{S}_{\mathrm{cor}}$, where the context points toward the correct answer. Empirical validation
against alternative signals is provided in Section~\ref{sec:gate_validation}.

\paragraph{Strength.}
A bounded signal $s_t \in [0,1]$ controls the contrastive magnitude. We use the normalized Jensen--Shannon divergence:
\begin{equation}
s_t = \mathrm{JSD}\bigl(
p_{\mathrm{ctx},t}\,\|\,p_{\mathrm{pri},t}
\bigr) / \log 2.
\label{eq:arr-strength}
\end{equation}
Unlike a fixed contrastive weight, $\tau$ varies continuously with per-token conflict: strong contrast where the two distributions sharply disagree, near-zero adjustment where they align, removing the need for a single hyperparameter to cover all positions. We adopt JSD rather than KL divergence because the gate already resolves directionality; the strength need only quantify how far both sides deviate from their midpoint without encoding a directional preference. JSD's symmetric formulation naturally fills this role.

\paragraph{Routing Rule.}
The gate and strength compose into a mixing coefficient and decoding distribution:
\begin{equation}
\tau_t = 1 + (2 d_t - 1)\, s_t, \label{eq:arr-tau}
\end{equation}
\begin{equation}
    q_{\tau_t,t}(y) \propto p_{\mathrm{pri},t}(y)^{1-\tau_t}\, p_{\mathrm{ctx},t}(y)^{\tau_t}.
\end{equation}
$d_t{=}1$ extrapolates ($\tau_t\in[1,2]$); $d_t{=}0$ interpolates ($\tau_t\in[0,1]$). Unlike existing methods restricted to $\tau \ge 1$, ARR can also pull back    
toward the prior ($\tau_t < 1$).
\section{Experiments and Results}
\subsection{Experimental Setup}
\label{sec:setup}

\paragraph{Datasets and Metrics.}
We evaluate on two sets of benchmarks. The first group consists of four established short-form QA datasets: Natural Questions (NQ;~\citealp{kwiatkowski2019natural}), TriviaQA~\citep{joshi2017triviaqa}, HotpotQA~\citep{yang2018hotpotqa}, and the tabular dataset TabMWP~\citep{lu2022dynamic}. Together, they represent a context-faithful QA setting where the gold context is trustworthy and parametric priors play a minimal role (NQ, TriviaQA), a multi-hop setting that requires aggregating evidence across documents (HotpotQA), and a structured-context setting that requires numerical reasoning over tables (TabMWP).

The second group is our TriState-Bench, which evaluates the three conflict states: correction ($\mathcal{S}_{\mathrm{cor}}$), resistance ($\mathcal{S}_{\mathrm{res}}$), and agreement ($\mathcal{S}_{\mathrm{agr}}$). Each state contains 400 benchmark samples. Figure~\ref{fig:prior_combined} reports the share of $\mathcal{P}_{\mathrm{right}}$, $\mathcal{P}_{\mathrm{uncertain}}$, and $\mathcal{P}_{\mathrm{wrong}}$ for each model on the same fact pool. The composition varies noticeably: $\mathcal{P}_{\mathrm{wrong}}$ ranges from 11.1\% to 20.3\%, and Person-related facts dominate both correct and incorrect priors, while Location facts are disproportionately well-known and Scientific Terms disproportionately unknown. These patterns confirm the necessity of our per-model prior calibration.

\begin{figure}[t]
\centering
  \includegraphics[width=\linewidth]{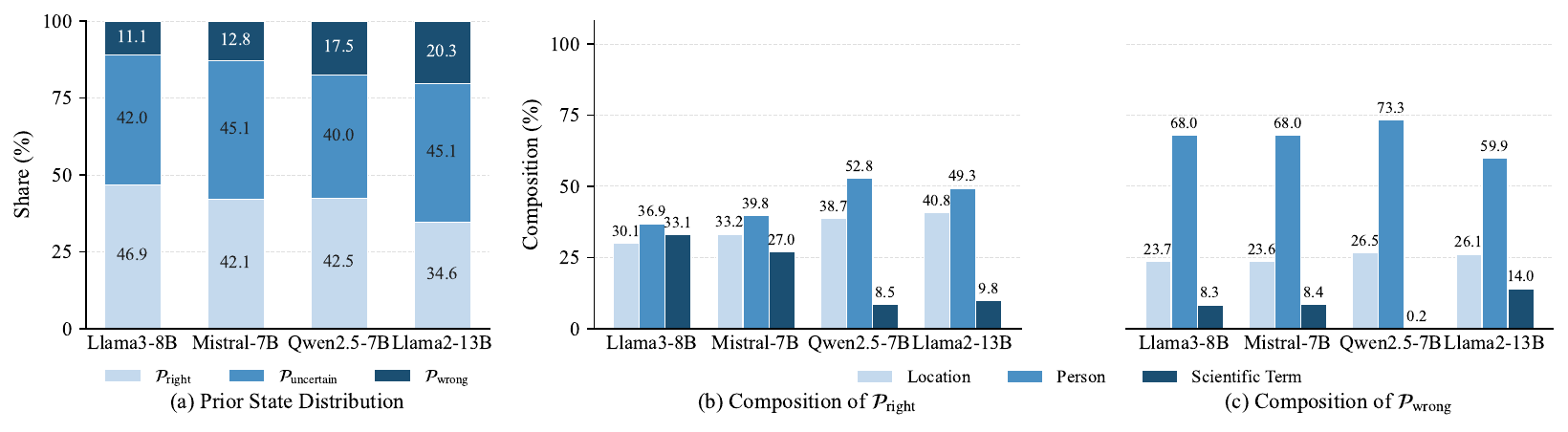}
  \caption{Distribution of prior knowledge states across four LLMs. (a) Prior state distribution showing the proportion of facts where each model holds correct ($\mathcal{P}_{\text{right}}$), uncertain  ($\mathcal{P}_{\text{uncertain}}$), or incorrect ($\mathcal{P}_{\text{wrong}}$) priors. (b–c) Entity-type composition of $\mathcal{P}_{\text{right}}$ and $\mathcal{P}_{\text{wrong}}$.}
  \label{fig:prior_combined}
\end{figure}

We report Exact Match (EM) and token-level F1 on all benchmarks. For additional details of QA datasets, refer to Appendix~\ref{app:qa-datasets}.

\paragraph{Source of Context.}
For NQ, TriviaQA, and HotpotQA we use the provided gold context; for TabMWP, the table is the context and the problem statement is the query. For TriState-Bench, the gold context is a web-retrieved evidence block grounded to the canonical answer, and the corrupted context swaps the gold span for a type-matched distractor. The examples of $(x, c)$ are shown in Table~\ref{tab:dataset_examples}; prompts are listed in Appendix~\ref{app:prompts}.

\paragraph{Models.}
We conduct experiments on four open-weight families, each in base and instruction-tuned variants: Llama2-13B~\citep{touvron2023llama}, Llama3-8B~\citep{grattafiori2024llama}, Mistral-7B~\citep{jiang2023mistral7b}, and Qwen2.5-7B~\citep{qwen2025qwen25technicalreport}.
Across this set, we first test whether the regime asymmetry of Corollary~\ref{cor:regime-state} manifests empirically, and then whether ARR remains robust across families and instruction tuning.

\paragraph{Baselines.}
We compare ARR against five test-time decoding baselines: Greedy Decoding; CAD~\citep{shi2024trusting} ($\alpha=1$); COIECD~\citep{yuan2024discerning} ($\lambda=0.25$, inner $\alpha=1$); AdaCAD~\citep{wang2025adacad} (JSD-driven $\alpha_t$); and CoCoA~\citep{khandelwal2025cocoa} (R\'enyi order $0.5$, peakedness weight $z=5.0$, entropy-gap weight $\gamma=1.0$).
To further probe how regime choice affects each conflict state, we also include Greedy-no-ctx, which decodes without context as a prior-only reference, together with sweeps of simple interpolation ($\tau\in\{0.25,0.5,0.75,1.0\}$) and tuned CAD ($\alpha\in\{0.25,0.5,0.75,1.0\}$).

\subsection{Main Results}

\begin{table*}[t]
\centering
\scriptsize
\setlength{\tabcolsep}{3.2pt}
\renewcommand{\arraystretch}{1.08}
\resizebox{0.98\textwidth}{!}{%
\begin{tabular}{ll*{6}{cc}}
\toprule
\multirow{2}{*}{Model} & \multirow{2}{*}{Method}
& \multicolumn{2}{c}{NQ}
& \multicolumn{2}{c}{TabMWP}
& \multicolumn{2}{c}{HotpotQA}
& \multicolumn{2}{c}{TriviaQA}
& \multicolumn{2}{c}{TriState}
& \multicolumn{2}{c}{\textbf{Avg.}} \\
\cmidrule(lr){3-4}
\cmidrule(lr){5-6}
\cmidrule(lr){7-8}
\cmidrule(lr){9-10}
\cmidrule(lr){11-12}
\cmidrule(lr){13-14}
& & EM & F1 & EM & F1 & EM & F1 & EM & F1 & EM & F1 & EM & F1 \\
\midrule

\multirow{6}{*}{\textbf{Llama3-8B}}
& Greedy & 31.61 & 47.68 & 29.50 & 35.31 & 18.56 & 33.06 & 54.40 & 64.46 & 42.92 & 51.49 & 35.40 & 46.40 \\
& CAD    & 14.76 & 30.74 & 12.60 & 16.40 & 5.44  & 15.65 & 18.55 & 27.46 & 14.58 & 28.70 & 13.19 & 23.79 \\
& COIECD & 22.26 & 39.81 & 19.10 & 23.04 & 12.05 & 25.93 & 33.95 & 43.39 & 22.33 & 34.87 & 21.94 & 33.41 \\
& AdaCAD & 28.33 & 44.93 & 29.70 & \textbf{35.34} & 16.84 & 31.08 & 51.25 & 61.63 & 34.42 & 44.76 & 32.11 & 43.55 \\
& CoCoA  & 22.23 & 38.99 & 25.10 & 30.66 & 10.56 & 23.59 & 39.00 & 48.86 & 20.75 & 33.70 & 23.53 & 35.16 \\
& \cellcolor{gg}\textbf{ARR(Ours)} & \textbf{36.37} & \textbf{51.58} & \textbf{30.00} & 35.16 & \textbf{19.03} & \textbf{33.11} & \textbf{56.10} & \textbf{66.50} & \textbf{61.67} & \textbf{68.14} & \cellcolor{gg}\textbf{40.63} & \cellcolor{gg}\textbf{50.90} \\
\midrule

\multirow{6}{*}{\textbf{Mistral-7B}}
& Greedy & \textbf{39.25} & \textbf{50.31} & 7.90  & 11.18 & \textbf{19.54} & 28.22 & 56.45 & 64.71 & 59.75 & 65.17 & 36.58 & 43.92 \\
& CAD    & 15.84 & 22.76 & 5.70  & 7.46  & 4.62  & 7.21  & 23.85 & 32.98 & 31.92 & 37.15 & 16.39 & 21.51 \\
& COIECD & 35.98 & 46.01 & \textbf{10.20} & \textbf{12.30} & 16.30 & 23.10 & 41.75 & 48.01 & 52.92 & 57.88 & 31.43 & 37.46 \\
& AdaCAD & 36.86 & 46.52 & 3.70  & 6.40  & 18.87 & 26.87 & 54.90 & 63.29 & 58.83 & 64.09 & 34.63 & 41.43 \\
& CoCoA  & 13.86 & 29.24 & 6.20  & 8.21  & 1.19  & 9.27  & 23.85 & 32.98 & 29.33 & 41.76 & 14.89 & 24.29 \\
& \cellcolor{gg}\textbf{ARR(Ours)} & 36.90 & 49.36 & 6.10  & 10.05 & 18.80 & \textbf{28.41} & \textbf{57.65} & \textbf{65.87} & \textbf{64.17} & \textbf{70.28} & \cellcolor{gg}\textbf{36.72} & \cellcolor{gg}\textbf{44.79} \\
\midrule

\multirow{6}{*}{\textbf{Qwen2.5-7B}}
& Greedy & \textbf{60.35} & \textbf{70.87} & 37.50 & 39.68 & 28.98 & \textbf{40.14} & 39.15 & 47.63 & 58.83 & 64.71 & \textbf{44.96} & \textbf{52.61} \\
& CAD    & 21.85 & 32.99 & 13.60 & 15.71 & 3.69  & 12.63 & 7.00  & 16.72 & 20.75 & 37.38 & 13.38 & 23.09 \\
& COIECD & 53.22 & 65.81 & 36.50 & 39.63 & 18.39 & 29.35 & 25.80 & 35.43 & 41.00 & 51.92 & 34.98 & 44.43 \\
& AdaCAD & 57.21 & 68.71 & \textbf{37.80} & \textbf{40.27} & 24.13 & 34.56 & 33.50 & 42.19 & 54.75 & 61.67 & 41.48 & 49.48 \\
& CoCoA  & 33.72 & 50.36 & 29.90 & 33.18 & 6.82  & 17.66 & 14.65 & 25.72 & 19.17 & 34.40 & 20.85 & 32.26 \\
& \cellcolor{gg}\textbf{ARR(Ours)} & 48.38 & 61.20 & 37.00 & 39.54 & \textbf{29.98} & 40.00 & \textbf{44.90} & \textbf{53.55} & \textbf{63.25} & \textbf{68.76} & \cellcolor{gg}44.70 & \cellcolor{gg}\textbf{52.61} \\
\midrule

\multirow{6}{*}{\textbf{Llama2-13B}}
& Greedy & 46.54 & 59.60 & 21.10 & 23.41 & \textbf{23.77} & 34.44 & \textbf{59.00} & \textbf{67.89} & 60.67 & 65.40 & 42.22 & 50.15 \\
& CAD    & 30.41 & 48.12 & 14.80 & 17.82 & 10.57 & 19.74 & 38.05 & 49.57 & 46.33 & 54.61 & 28.03 & 37.97 \\
& COIECD & 44.94 & 60.42 & 18.70 & 22.12 & 21.05 & 32.21 & 49.20 & 60.06 & 55.00 & 60.59 & 37.78 & 47.08 \\
& AdaCAD & \textbf{47.80} & \textbf{60.86} & \textbf{21.70} & \textbf{24.51} & 23.61 & \textbf{34.49} & 58.45 & 67.45 & 60.08 & 64.96 & \textbf{42.33} & \textbf{50.45} \\
& CoCoA  & 43.02 & 58.36 & 18.50 & 21.55 & 18.50 & 29.08 & 51.55 & 61.92 & 55.75 & 61.34 & 37.46 & 46.45 \\
& \cellcolor{gg}\textbf{ARR(Ours)} & 43.43 & 57.19 & 21.20 & 22.88 & 23.08 & 33.69 & 58.95 & 67.62 & \textbf{64.92} & \textbf{69.37} & \cellcolor{gg}42.32 & \cellcolor{gg}50.15 \\

\bottomrule
\end{tabular}%
}
\caption{Performance comparison across benchmarks. Each benchmark reports EM and F1. The \textbf{Avg.}\ column is the arithmetic mean over all five benchmarks (NQ, TabMWP, HotpotQA, TriviaQA, and TriState-Bench). Bold marks the best value within each model block.}
\label{tab:qa_result}
\end{table*}

\paragraph{Main results.}
Table~\ref{tab:qa_result} reports performance on the four standard QA benchmarks together with TriState-Bench. The standard QA columns test whether a conflict-aware decoder remains safe when the resistance subset $\mathcal{S}_{\mathrm{res}}$ is rare; TriState-Bench measures aggregate conflict-resolution ability.

On standard QA, existing context-aware baselines regress substantially relative to Greedy decoding: CAD drops Llama3-8B NQ from $31.61$ to $14.76$ EM, and CoCoA reduces Qwen2.5-7B HotpotQA from $28.98$ to $6.82$. Both failures instantiate the over-generation mode predicted by Corollary~\ref{cor:generation-mode}: contrastive extrapolation pushes log-odds past $p_{\mathrm{ctx},t}$ even when the context is already correct. ARR matches or exceeds Greedy on most dataset--model pairs and is the best contrastive decoder on average for Llama3-8B, Mistral-7B, and Llama2-13B, remaining within $0.26$ EM of Greedy on Qwen2.5-7B. No other context-aware method avoids regression against Greedy across all four models. On TriState-Bench, ARR leads every baseline for all four base models.

\paragraph{TriState-Bench decomposition.}
Table~\ref{tab:tristate_result} decomposes TriState-Bench into correction ($\mathcal{S}_{\mathrm{cor}}$), resistance ($\mathcal{S}_{\mathrm{res}}$), and agreement ($\mathcal{S}_{\mathrm{agr}}$). All context-aware baselines collapse on $\mathcal{S}_{\mathrm{res}}$: across the four models, every contrastive method scores below $3.5$ EM, with CAD and CoCoA below $1$. This is the empirical signature of regime asymmetry: a strength $\tau>1$ that corrects on $\mathcal{S}_{\mathrm{cor}}$ simultaneously amplifies the wrong preference on $\mathcal{S}_{\mathrm{res}}$. ARR is the only method that recovers $\mathcal{S}_{\mathrm{res}}$, lifting EM to $15.75$--$33.25$ across models, an order of magnitude above the strongest baseline. The gain stems from switching the gate to $\tau<1$ whenever $p_{\mathrm{ctx},t}$ is less committed than $p_{\mathrm{pri},t}$ (Theorem~\ref{thm:interpolation}). Crucially, this does not sacrifice $\mathcal{S}_{\mathrm{cor}}$ or $\mathcal{S}_{\mathrm{agr}}$: ARR is best or near-best on both subsets across all four models. More results appear in Tables~\ref{tab:tristate_result_full_1}, ~\ref{tab:tristate_result_full_2} and Appendix~\ref{app:instruct-experiments}.

\begin{table*}[t]
\centering
\small
\resizebox{0.72\textwidth}{!}{%
\begin{tabular}{ll*{3}{cc}}
\toprule
\multirow{2}{*}{Model} & \multirow{2}{*}{Method}
& \multicolumn{2}{c}{$\mathcal{S}_{\mathrm{cor}}$}
& \multicolumn{2}{c}{$\mathcal{S}_{\mathrm{res}}$}
& \multicolumn{2}{c}{$\mathcal{S}_{\mathrm{agr}}$} \\
\cmidrule(lr){3-4}
\cmidrule(lr){5-6}
\cmidrule(lr){7-8}
& & EM & F1 & EM & F1 & EM & F1 \\
\midrule

\multirow{6}{*}{\textbf{Llama3-8B}}
& Greedy & 59.50 & 70.79 & 4.50  & 9.78  & 64.75 & 73.92 \\
& CAD    & 21.75 & 41.07 & 0.75  & 7.11  & 21.25 & 37.92 \\
& COIECD & 30.25 & 47.35 & 1.75  & 6.29  & 35.00 & 50.99 \\
& AdaCAD & 51.50 & 64.19 & 3.25  & 9.67  & 48.50 & 60.43 \\
& CoCoA  & 31.25 & 47.88 & 1.00  & 7.68  & 30.00 & 45.52 \\
 & \cellcolor{gg}\textbf{ARR(Ours)} & \textbf{75.00} & \textbf{82.68} & \textbf{33.25} & \textbf{38.59} & \textbf{76.75} & \textbf{83.15}\\
\midrule

\multirow{6}{*}{\textbf{Mistral-7B}}
& Greedy & \textbf{91.25} & \textbf{94.75} & 5.25  & 12.99 & 82.75 & 87.76 \\
& CAD    & 70.50 & 78.48 & 0.25  & 5.47  & 25.00 & 27.51 \\
& COIECD & 85.00 & 89.46 & 1.00  & 8.61  & 72.75 & 75.56 \\
& AdaCAD & 91.00 & 94.25 & 2.25  & 10.06 & \textbf{83.25} & \textbf{87.95} \\
& CoCoA  & 55.50 & 70.10 & 0.50  & 6.98  & 32.00 & 48.19 \\
&  \cellcolor{gg}\textbf{ARR(Ours)} & 88.00 & 92.73 & \textbf{22.50} & \textbf{30.66} & 82.00 & 87.46 \\
\midrule

\multirow{6}{*}{\textbf{Qwen2.5-7B}}
& Greedy & \textbf{89.50} & \textbf{93.32} & 1.00  & 10.04 & 86.00 & 90.76 \\
& CAD    & 25.75 & 49.65 & 0.25  & 6.30  & 36.25 & 56.18 \\
& COIECD & 65.75 & 76.63 & 0.50  & 7.19  & 56.75 & 71.96 \\
& AdaCAD & 81.25 & 87.50 & 0.75  & 8.94  & 82.25 & 88.56 \\
& CoCoA  & 26.50 & 46.41 & 0.25  & 6.19  & 30.75 & 50.60 \\
&  \cellcolor{gg}\textbf{ARR(Ours)} & 85.75 & 91.09 & \textbf{15.75} & \textbf{22.92} & \textbf{88.25} & \textbf{92.27} \\
\midrule

\multirow{6}{*}{\textbf{Llama2-13B}}
& Greedy & \textbf{86.50} & \textbf{90.79} & 1.75  & 9.58  & \textbf{93.75} & \textbf{95.83} \\
& CAD    & 66.00 & 75.74 & 0.50  & 7.75  & 72.50 & 80.35 \\
& COIECD & 80.00 & 84.91 & 1.00  & 8.62  & 84.00 & 88.25 \\
& AdaCAD & 85.25 & 89.83 & 1.75  & 9.65  & 93.25 & 95.40 \\
& CoCoA  & 80.25 & 85.44 & 1.50  & 9.13  & 85.50 & 89.46 \\
&  \cellcolor{gg}\textbf{ARR(Ours)} & 84.25 & 89.16 & \textbf{17.00} & \textbf{23.21} & \textbf{93.75} & 95.81 \\

\bottomrule
\end{tabular}%
}
\caption{Performance on the three TriState-Bench subsets. $\mathcal{S}_{\mathrm{cor}}$ (Correction): gold context, prior incorrect; $\mathcal{S}_{\mathrm{res}}$ (Resistance): corrupted context, prior correct; $\mathcal{S}_{\mathrm{agr}}$ (Agreement): gold context, prior correct. ARR consistently dominates the resistance subset $\mathcal{S}_{\mathrm{res}}$ where all baselines collapse, while staying competitive on $\mathcal{S}_{\mathrm{cor}}$ and $\mathcal{S}_{\mathrm{agr}}$. Bold marks the best value within each model block.}
\label{tab:tristate_result}
\end{table*}

\subsection{Gate Validation}
\label{sec:gate_validation}
The gate (Eq.~\ref{eq:arr-gate}) routes between interpolation and extrapolation based on whether the context is more committed than the prior. We validate this choice by comparing four candidate signals, grouped by whether they carry directional information:

\paragraph{Magnitude-only (detect conflict existence).}
\begin{itemize}
  \item \textbf{A}: $\arg\max p_{\text{ctx}} \neq \arg\max p_{\text{pri}}$ (top-1 token differs).
  \item \textbf{B}: $\text{JSD}(p_{\text{ctx}}\,\Vert\,p_{\text{pri}}) > 0.5$ (distributions are sharply divergent).
\end{itemize}

\paragraph{Confidence-asymmetry (resolve conflict direction).}
\begin{itemize}
  \item \textbf{C}: $H(p_{\text{pri}}) > H(p_{\text{ctx}})$ (context sharpens the distribution).
  \item \textbf{D}: $\max p_{\text{ctx}} > \max p_{\text{pri}}$ (context raises top-1 confidence).
\end{itemize}

\begin{wrapfigure}{r}{0.48\textwidth}
  \centering
  \vspace{-12pt}
  \includegraphics[width=0.48\textwidth]{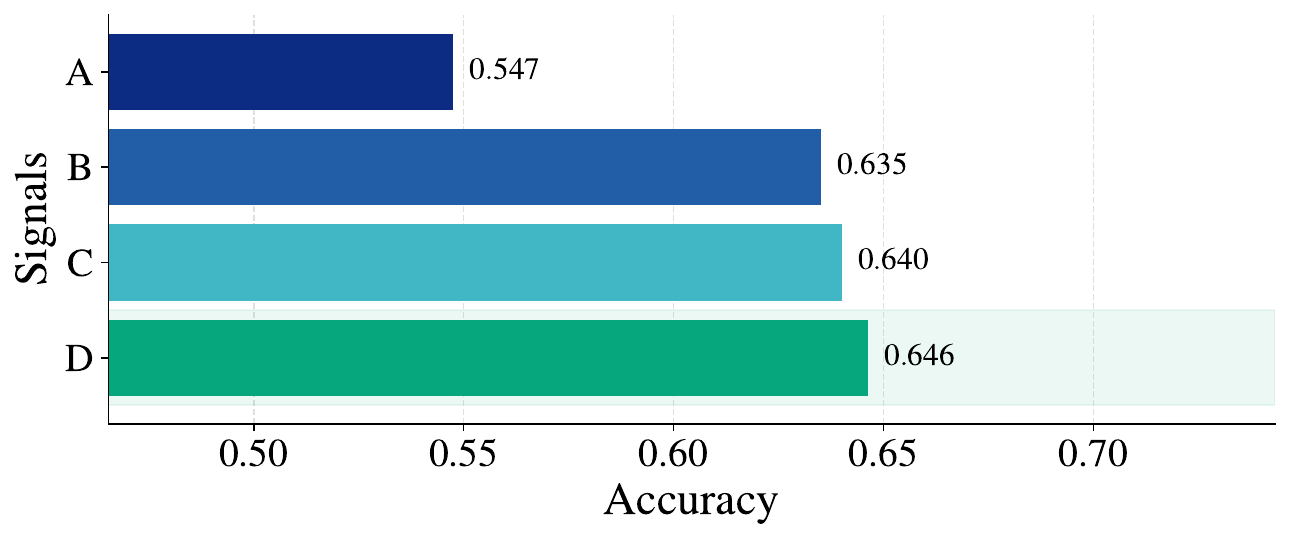}
  \caption{Gate accuracy of each candidate signal. Confidence-asymmetry signals (C, D) consistently outperform magnitude-only signals (A, B), confirming that directionality is necessary for regime separation.}
  \label{fig:signal_comparison}
  \vspace{-8pt}
\end{wrapfigure}

Magnitude-only signals fire whenever $p_{\text{pri},t}$ and $p_{\text{ctx},t}$ diverge but cannot determine which side is more reliable; signal A, in particular, barely exceeds the chance rate of 0.5. Confidence-asymmetry signals additionally resolve direction: a positive value indicates that the context concentrates more mass on its top prediction than the prior does. Because $\mathcal{S}_{\mathrm{cor}}$ and $\mathcal{S}_{\mathrm{res}}$ differ precisely in this asymmetry (Corollary~\ref{cor:regime-state}), only directional signals can reliably separate the two regimes. Figure~\ref{fig:signal_comparison} confirms this: both C and D surpass B by a clear margin, while A hovers near chance, indicating that detecting conflict alone is insufficient without resolving its direction.

\subsection{Ablation Study}
\begin{wraptable}{r}{0.48\textwidth}
\centering
\small
\setlength{\tabcolsep}{3.5pt}
\renewcommand{\arraystretch}{1.05}
\begin{tabular}{@{}l@{\hspace{3pt}}ccc@{}}
\toprule
Method & \textbf{Trad QA} & \textbf{TriState} & \textbf{Avg.} \\
\midrule
Greedy & \textbf{41.5}/\textbf{49.6} & 58.8/64.7 & \textbf{45.0}/\textbf{52.6} \\
CAD & 11.5/19.5 & 20.8/37.4 & 13.4/23.1 \\
COIECD & 33.5/42.6 & 41.0/51.9 & 35.0/44.4 \\
AdaCAD & 38.2/46.4 & 54.8/61.7 & 41.5/49.5 \\
CoCoA & 21.3/31.7 & 19.2/34.4 & 20.9/32.3 \\
\midrule
\textbf{ARR-JS (Ours)} & 40.1/48.6 & \textbf{63.3}/\textbf{68.8} & 44.7/\textbf{52.6} \\
ARR-KL & 35.8/44.1 & 62.1/68.0 & 41.0/48.9 \\
ARR-D & 37.8/46.3 & 59.9/65.8 & 42.2/50.2 \\
\bottomrule
\end{tabular}
\caption{Ablation results on strength on Qwen2.5-7B (EM/F1).}
\label{tab:ablation_result}
\vspace{-0.5em}
\end{wraptable}
Having validated the gate signal in Section~\ref{sec:gate_validation}, we fix $d_t = \mathbb{1}[\max p_{\mathrm{ctx},t} > \max p_{\mathrm{pri},t}]$ and vary the strength function $s_t$ among three choices: normalized JSD (ARR-JS, our default), $1 - \exp(-\mathrm{KL})$ (ARR-KL), and a constant $s_t \equiv 0.5$ (ARR-D). As shown in Table~\ref{tab:ablation_result}, all three variants surpass every baseline on TriState-Bench; even ARR-D reaches $59.9$ EM, above AdaCAD ($54.8$). The gain on $\mathcal{S}_{\mathrm{res}}$ therefore stems from the gate direction rather than the strength formula, confirming that explicitly routing between interpolation and extrapolation matters more than the magnitude of the adjustment. On Trad QA, however, the formula does matter: ARR-JS stays within $1.4$ EM of Greedy ($40.1$ vs.\ $41.5$), the smallest degradation among all contrastive methods, whereas ARR-KL and ARR-D drop by $5.7$ and $3.7$ EM respectively. Balancing TriState-Bench gains with minimal Trad QA regression, we adopt ARR-JS as the default.

\subsection{Case Study: Cross-Model Commonalities and Divergences}
\label{sec:case-study}

\begin{figure*}[t!]
\centering
  \includegraphics[width=\linewidth]{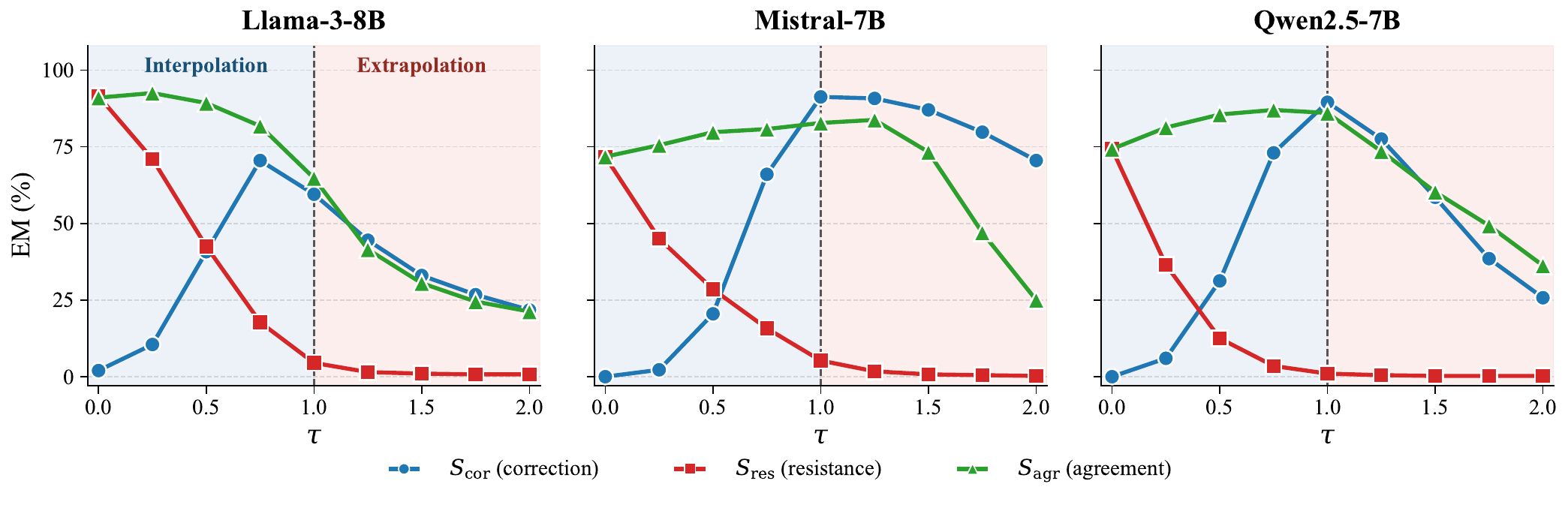}
  \caption{Blue/red shading separates the interpolation ($\tau\in[0,1]$) and extrapolation ($\tau>1$) regimes. $\mathcal{S}_{\mathrm{res}}$ decays monotonically with $\tau$; $\mathcal{S}_{\mathrm{cor}}$ peaks near $\tau\approx 1$ and collapses under extrapolation, with model-dependent severity.}
  \label{fig:threshold_change}
\end{figure*}

Figure~\ref{fig:threshold_change} plots EM against $\tau$ on the full TriState-Bench for Llama3-8B, Mistral-7B and Qwen2.5-7B (more results in Figure~\ref{fig:threshold_change_full}; detailed cases in Figure~\ref{fig:detailed_cases}). The empirical patterns divide into two cross-model commonalities and one class of model-specific divergences.

\paragraph{Commonality 1: threshold reversal within interpolation.}
On every model, $\mathcal{S}_{\mathrm{cor}}$ EM rises sharply within the interpolation regime rather than climbing gradually. Llama3-8B jumps from 2\% at $\tau{=}0$ to 71\% at $\tau{=}0.75$; Mistral-7B from 0\% to 66\%; Qwen2.5-7B from 0\% to 73\%; Llama3-Instruct from 0\% to 92\%. Symmetrically, $\mathcal{S}_{\mathrm{res}}$ EM drops steeply over the same interval (Llama3-8B: 92\%$\to$18\%; Mistral: 72\%$\to$16\%; Qwen: 74\%$\to$4\%; Instruct: 82\%$\to$5\%).

This pattern is the aggregate signature of Corollary~\ref{cor:regime-state}. Each sample carries a pairwise reversal threshold $\tau^\star_{a,b} \in (0,1)$ at which the decoded distribution flips its preference between the correct token $a$ and the distractor $b$. As $\tau$ sweeps upward, progressively more $\mathcal{S}_{\mathrm{cor}}$ samples cross their individual $\tau^\star$ and switch from wrong to right, while progressively more $\mathcal{S}_{\mathrm{res}}$ samples cross theirs and switch from right to wrong. The steepness of the curves reflects the concentration of $\tau^\star$ values: on Llama3-Instruct, most thresholds cluster near $\tau \approx 0.5$, producing an almost step-function rise. No single $\tau$ simultaneously resolves all samples, confirming that interpolation exposes an irreducible trade-off between correction power and resistance preservation.

\paragraph{Commonality 2: structural collapse of $\mathcal{S}_{\mathrm{res}}$ under extrapolation.}
Once $\tau$ crosses 1, $\mathcal{S}_{\mathrm{res}}$ EM drops to near zero on every model and stays flat across the entire extrapolation interval: Llama3-8B holds 0.75\% from $\tau{=}1.25$ to $\tau{=}2$; Mistral-7B drops from 1.75\% to 0.25\%; Qwen2.5-7B from 0.50\% to 0.25\%; Llama3-Instruct at 0.25\% throughout. This is not gradual degradation but a structural cliff: the transition from $\tau{=}0.75$ to $\tau{=}1.25$ collapses $\mathcal{S}_{\mathrm{res}}$ by 15--20$\times$ on every model.

Corollary~\ref{cor:regime-state} explains why. In $\mathcal{S}_{\mathrm{res}}$, the prior supports the correct token ($\ell^{\mathrm{prior}}_{a,b} > 0$) while the context favors the distractor ($\ell^{\mathrm{ctx}}_{a,b} < 0$), placing the reversal threshold at $\tau^\star \in (0,1)$. By $\tau{=}1$ most samples have already flipped; extrapolation then drives the pairwise log-odds further negative without bound, locking in the wrong answer irreversibly. The collapse is not a failure of any particular method but a structural consequence of the power family: \emph{any} member with $\tau > 1$ operates past the point of no return for resistance-state samples.

The canonical instance reproduces on all four models. On \emph{What is the capital of France?} with misleading context \emph{Lyon}, the prior-only output ($\tau{=}0$) returns \emph{Paris}, but greedy decoding ($\tau{=}1$) and every extrapolation method (CAD, AdaCAD, CoCoA, COIECD) return \emph{Lyon}. The same flip recurs on \emph{first president of the United States} (Washington $\to$ Adams), \emph{largest ocean} (Pacific $\to$ Atlantic), and \emph{painter of the Mona Lisa} (Leonardo $\to$ Raphael).

\paragraph{Divergence: one mechanism, three generation-mode failures.}
Corollary~\ref{cor:generation-mode} predicts that extrapolation degrades generation through three routes, depending on where the prior concentrates mass: continuation tokens, over stop tokens, or low-probability noise tokens. Figure~\ref{fig:detailed_failure_cases} illustrates representative cases for each failure mode.

\begin{itemize}
    \item \textbf{Llama3-8B: over-generation.} $\mathcal{S}_{\mathrm{agr}}$ EM drops from 82\% at $\tau{=}0.75$ to 21\% at $\tau{=}2$. On \emph{Which country has the most pyramids?} (gold: \emph{Sudan}), interpolation at $\tau{=}0.75$ outputs a clean \emph{Sudan}, while extrapolation at $\tau{=}2$ copies verbatim from the context: \emph{``Sudan has the most pyramids of any country in the world, with approximately 200 to 255 known pyramids\ldots''}. Once the correct entity has been emitted, the context marginally favors continuation over EOS; as $\tau$ crosses into extrapolation, the negative exponent amplifies this preference, and Llama3-8B's weak prior on EOS provides no counterforce, yielding runaway continuation that worsens with $\tau$ (Figure~\ref{fig:gen_length_all}).
    
    \item \textbf{Qwen2.5-7B: early stopping.} $\mathcal{S}_{\mathrm{agr}}$ EM drops from 87\% at $\tau{=}0.75$ to 36\% at $\tau{=}2$. Extrapolation truncates multi-token answers at their first constituent: \emph{Greenland shark}$\to$\emph{Greenland}, \emph{Alexander Fleming}$\to$\emph{Alexander}, \emph{Sargasso Sea}$\to$\emph{Sargasso}. Qwen's prior places strong mass on EOS and sentence-final punctuation; as $\tau$ increases past~1, the same negative exponent amplifies that mass, triggering premature termination.
    
    \item \textbf{Mistral-7B: distribution collapse.} $\mathcal{S}_{\mathrm{agr}}$ EM drops from 81\% at $\tau{=}0.75$ to 25\% at $\tau{=}2$. As $\tau$ grows, outputs degenerate into repeated underscore strings (\emph{``\_\_\_\_\_\_\_\_...''}), training-corpus residue such as scraped boilerplate (\emph{``\textcopyright~BrainMass Inc.\ brainmass.com October 10, 2019, 1:00 am ad1''}), or bare instruction fragments (\emph{``Instructions:''}). The negative exponent amplifies tokens to which the prior assigns negligible mass; Mistral's residual probability on filler and template fragments from its training corpus is enough for extrapolation to redirect generation toward precisely these tokens.
    \end{itemize}

All three failures trace to one mechanism: as $\tau$ enters the extrapolation regime ($\tau > 1$), the power family drives pairwise log-odds past the context endpoint while the negative exponent reshapes the prior's tail. Which extreme manifests depends on where each model's prior concentrates its residual mass---continuation tokens (Llama3), stop tokens (Qwen), or low-probability noise (Mistral). Corollary~\ref{cor:generation-mode} guarantees that extrapolation degrades generation quality; the prior's structural bias determines the failure mode.

\begin{figure}[t]
    \includegraphics[width=\linewidth]{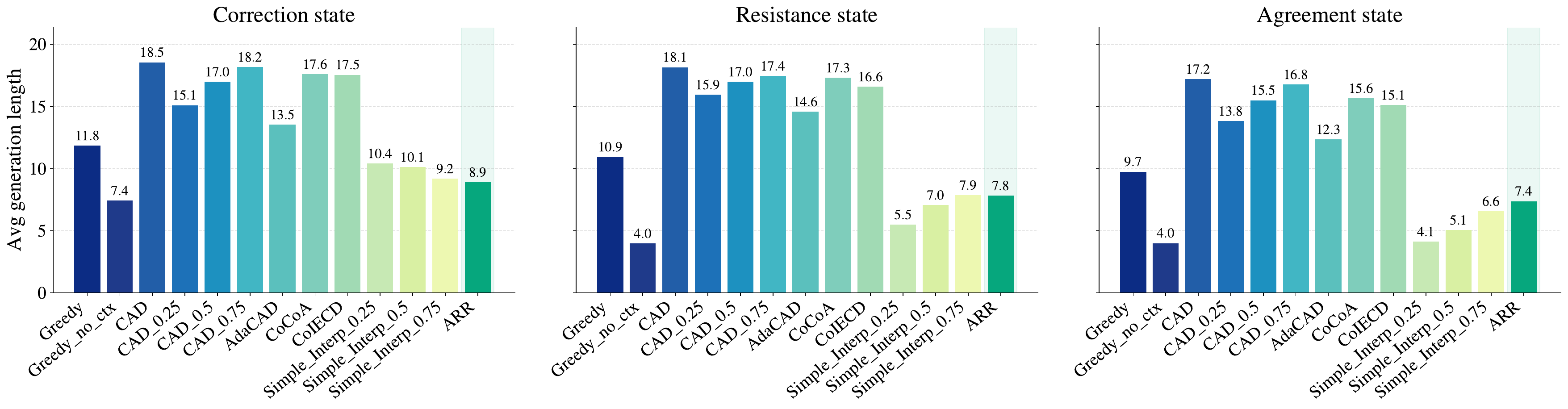}
    \caption{Generation length of Llama3-8B across all methods in correction, resistance, and agreement states.}
    \label{fig:gen_length_all}
\end{figure}

\section{Conclusion}

We generalize existing contrastive decoding methods into a power family $q_{\tau,t} \propto p_{\mathrm{pri},t}^{1-\tau}\,p_{\mathrm{ctx},t}^{\tau}$ and show that the family partitions at $\tau=1$ into interpolation and extrapolation: two structurally distinct regimes. A regime asymmetry emerges: extrapolation amplifies errors unboundedly when the prior is already correct, interpolation under-corrects when the context is correct, and no single static regime covers both directions.
To address this asymmetry, Adaptive Regime Routing (ARR) routes between regimes at each decoding step via a confidence-asymmetry gate, shifting the design question from how much to amplify the context to which side deserves authority, without introducing additional hyperparameters.
To make the gains from routing separately measurable, TriState-Bench conditions conflict states on each model's actual prior knowledge, enabling correction, resistance, and agreement to be independently observed.
However, the current gate relies on a single statistical scalar to estimate context credibility. Extending ARR to multi-source and multi-turn settings, where credibility signals accumulate across retrieval rounds, is a natural next step.
\section*{Limitations}
This work has three limitations.
\emph{Logit access}: routing signals are derived from the output distribution, so ARR does not apply to black-box APIs that expose only generated text (e.g., GPT-4, Claude).
\emph{Coarse credibility estimation}: the gate compresses context trustworthiness into a binary decision from a single scalar (max-probability gap), without modeling source reliability or evidential consistency. This work focuses on establishing the structural necessity of bidirectional regime switching; a more expressive credibility estimator that enables finer-grained routing is a direct next step.
\emph{Language coverage}: experiments are conducted exclusively on English QA; cross-lingual and code-mixed settings remain unevaluated.
\section*{Acknowledgments}
This work was supported by Alibaba Group through Alibaba Research Intern Program.

\clearpage
\bibliographystyle{colm2024_conference}
\bibliography{main}

\begin{thebibliography}{26}
\providecommand{\natexlab}[1]{#1}
\providecommand{\url}[1]{\texttt{#1}}
\expandafter\ifx\csname urlstyle\endcsname\relax
  \providecommand{\doi}[1]{doi: #1}\else
  \providecommand{\doi}{doi: \begingroup \urlstyle{rm}\Url}\fi

\bibitem[Grattafiori et~al.(2024)Grattafiori, Dubey, Jauhri, Pandey, Kadian, Al-Dahle, Letman, Mathur, Schelten, Vaughan, et~al.]{grattafiori2024llama}
Aaron Grattafiori, Abhimanyu Dubey, Abhinav Jauhri, Abhinav Pandey, Abhishek Kadian, Ahmad Al-Dahle, Aiesha Letman, Akhil Mathur, Alan Schelten, Alex Vaughan, et~al.
\newblock The llama 3 herd of models.
\newblock \emph{arXiv preprint arXiv:2407.21783}, 2024.

\bibitem[Jiang et~al.(2023)Jiang, Sablayrolles, Mensch, Bamford, Chaplot, de~las Casas, Bressand, Lengyel, Lample, Saulnier, Lavaud, Lachaux, Stock, Scao, Lavril, Wang, Lacroix, and Sayed]{jiang2023mistral7b}
Albert~Q. Jiang, Alexandre Sablayrolles, Arthur Mensch, Chris Bamford, Devendra~Singh Chaplot, Diego de~las Casas, Florian Bressand, Gianna Lengyel, Guillaume Lample, Lucile Saulnier, Lélio~Renard Lavaud, Marie-Anne Lachaux, Pierre Stock, Teven~Le Scao, Thibaut Lavril, Thomas Wang, Timothée Lacroix, and William~El Sayed.
\newblock Mistral 7b, 2023.
\newblock URL \url{https://arxiv.org/abs/2310.06825}.

\bibitem[Joshi et~al.(2017)Joshi, Choi, Weld, and Zettlemoyer]{joshi2017triviaqa}
Mandar Joshi, Eunsol Choi, Daniel~S Weld, and Luke Zettlemoyer.
\newblock Triviaqa: A large scale distantly supervised challenge dataset for reading comprehension.
\newblock In \emph{Proceedings of the 55th Annual Meeting of the Association for Computational Linguistics (Volume 1: Long Papers)}, pp.\  1601--1611, 2017.

\bibitem[Kasai et~al.(2023)Kasai, Sakaguchi, Le~Bras, Asai, Yu, Radev, Smith, Choi, Inui, et~al.]{kasai2023realtime}
Jungo Kasai, Keisuke Sakaguchi, Ronan Le~Bras, Akari Asai, Xinyan Yu, Dragomir Radev, Noah~A Smith, Yejin Choi, Kentaro Inui, et~al.
\newblock Realtime qa: What's the answer right now?
\newblock \emph{Advances in neural information processing systems}, 36:\penalty0 49025--49043, 2023.

\bibitem[Khandelwal et~al.(2025)Khandelwal, Gupta, and Agrawal]{khandelwal2025cocoa}
Anant Khandelwal, Manish Gupta, and Puneet Agrawal.
\newblock Cocoa: Confidence-and context-aware adaptive decoding for resolving knowledge conflicts in large language models.
\newblock In \emph{Proceedings of the 2025 Conference on Empirical Methods in Natural Language Processing}, pp.\  6846--6866, 2025.

\bibitem[Kwiatkowski et~al.(2019)Kwiatkowski, Palomaki, Redfield, Collins, Parikh, Alberti, Epstein, Polosukhin, Devlin, Lee, et~al.]{kwiatkowski2019natural}
Tom Kwiatkowski, Jennimaria Palomaki, Olivia Redfield, Michael Collins, Ankur Parikh, Chris Alberti, Danielle Epstein, Illia Polosukhin, Jacob Devlin, Kenton Lee, et~al.
\newblock Natural questions: a benchmark for question answering research.
\newblock \emph{Transactions of the Association for Computational Linguistics}, 7:\penalty0 453--466, 2019.

\bibitem[Lewis et~al.(2020)Lewis, Perez, Piktus, Petroni, Karpukhin, Goyal, K{\"u}ttler, Lewis, Yih, Rockt{\"a}schel, et~al.]{lewis2020retrieval}
Patrick Lewis, Ethan Perez, Aleksandra Piktus, Fabio Petroni, Vladimir Karpukhin, Naman Goyal, Heinrich K{\"u}ttler, Mike Lewis, Wen-tau Yih, Tim Rockt{\"a}schel, et~al.
\newblock Retrieval-augmented generation for knowledge-intensive nlp tasks.
\newblock \emph{Advances in neural information processing systems}, 33:\penalty0 9459--9474, 2020.

\bibitem[Li et~al.(2023{\natexlab{a}})Li, Rawat, Zaheer, Wang, Lukasik, Veit, Yu, and Kumar]{li2023large}
Daliang Li, Ankit~Singh Rawat, Manzil Zaheer, Xin Wang, Michal Lukasik, Andreas Veit, Felix Yu, and Sanjiv Kumar.
\newblock Large language models with controllable working memory.
\newblock In \emph{Findings of the association for computational linguistics: ACL 2023}, pp.\  1774--1793, 2023{\natexlab{a}}.

\bibitem[Li et~al.(2025)Li, Chen, and Tong]{li2025taming}
Gaotang Li, Yuzhong Chen, and Hanghang Tong.
\newblock Taming knowledge conflicts in language models.
\newblock \emph{arXiv preprint arXiv:2503.10996}, 2025.

\bibitem[Li et~al.(2023{\natexlab{b}})Li, Holtzman, Fried, Liang, Eisner, Hashimoto, Zettlemoyer, and Lewis]{li2023contrastive}
Xiang~Lisa Li, Ari Holtzman, Daniel Fried, Percy Liang, Jason Eisner, Tatsunori~B Hashimoto, Luke Zettlemoyer, and Mike Lewis.
\newblock Contrastive decoding: Open-ended text generation as optimization.
\newblock In \emph{Proceedings of the 61st annual meeting of the association for computational linguistics (volume 1: Long papers)}, pp.\  12286--12312, 2023{\natexlab{b}}.

\bibitem[Longpre et~al.(2021)Longpre, Perisetla, Chen, Ramesh, DuBois, and Singh]{longpre2021entity}
Shayne Longpre, Kartik Perisetla, Anthony Chen, Nikhil Ramesh, Chris DuBois, and Sameer Singh.
\newblock Entity-based knowledge conflicts in question answering.
\newblock In \emph{Proceedings of the 2021 conference on empirical methods in natural language processing}, pp.\  7052--7063, 2021.

\bibitem[Lu et~al.(2022)Lu, Qiu, Chang, Wu, Zhu, Rajpurohit, Clark, and Kalyan]{lu2022dynamic}
Pan Lu, Liang Qiu, Kai-Wei Chang, Ying~Nian Wu, Song-Chun Zhu, Tanmay Rajpurohit, Peter Clark, and Ashwin Kalyan.
\newblock Dynamic prompt learning via policy gradient for semi-structured mathematical reasoning.
\newblock \emph{arXiv preprint arXiv:2209.14610}, 2022.

\bibitem[Mallen et~al.(2023)Mallen, Asai, Zhong, Das, Khashabi, and Hajishirzi]{mallen2023not}
Alex Mallen, Akari Asai, Victor Zhong, Rajarshi Das, Daniel Khashabi, and Hannaneh Hajishirzi.
\newblock When not to trust language models: Investigating effectiveness of parametric and non-parametric memories.
\newblock In \emph{Proceedings of the 61st annual meeting of the association for computational linguistics (volume 1: Long papers)}, pp.\  9802--9822, 2023.

\bibitem[Nakano et~al.(2021)Nakano, Hilton, Balaji, Wu, Ouyang, Kim, Hesse, Jain, Kosaraju, Saunders, et~al.]{nakano2021webgpt}
Reiichiro Nakano, Jacob Hilton, Suchir Balaji, Jeff Wu, Long Ouyang, Christina Kim, Christopher Hesse, Shantanu Jain, Vineet Kosaraju, William Saunders, et~al.
\newblock Webgpt: Browser-assisted question-answering with human feedback.
\newblock \emph{arXiv preprint arXiv:2112.09332}, 2021.

\bibitem[Petroni et~al.(2019)Petroni, Rockt{\"a}schel, Riedel, Lewis, Bakhtin, Wu, and Miller]{petroni2019language}
Fabio Petroni, Tim Rockt{\"a}schel, Sebastian Riedel, Patrick Lewis, Anton Bakhtin, Yuxiang Wu, and Alexander Miller.
\newblock Language models as knowledge bases?
\newblock In \emph{Proceedings of the 2019 conference on empirical methods in natural language processing and the 9th international joint conference on natural language processing (EMNLP-IJCNLP)}, pp.\  2463--2473, 2019.

\bibitem[Qwen et~al.(2025)Qwen, :, Yang, Yang, Zhang, Hui, Zheng, Yu, Li, Liu, Huang, Wei, Lin, Yang, Tu, Zhang, Yang, Yang, Zhou, Lin, Dang, Lu, Bao, Yang, Yu, Li, Xue, Zhang, Zhu, Men, Lin, Li, Tang, Xia, Ren, Ren, Fan, Su, Zhang, Wan, Liu, Cui, Zhang, and Qiu]{qwen2025qwen25technicalreport}
Qwen, :, An~Yang, Baosong Yang, Beichen Zhang, Binyuan Hui, Bo~Zheng, Bowen Yu, Chengyuan Li, Dayiheng Liu, Fei Huang, Haoran Wei, Huan Lin, Jian Yang, Jianhong Tu, Jianwei Zhang, Jianxin Yang, Jiaxi Yang, Jingren Zhou, Junyang Lin, Kai Dang, Keming Lu, Keqin Bao, Kexin Yang, Le~Yu, Mei Li, Mingfeng Xue, Pei Zhang, Qin Zhu, Rui Men, Runji Lin, Tianhao Li, Tianyi Tang, Tingyu Xia, Xingzhang Ren, Xuancheng Ren, Yang Fan, Yang Su, Yichang Zhang, Yu~Wan, Yuqiong Liu, Zeyu Cui, Zhenru Zhang, and Zihan Qiu.
\newblock Qwen2.5 technical report, 2025.
\newblock URL \url{https://arxiv.org/abs/2412.15115}.

\bibitem[Roberts et~al.(2020)Roberts, Raffel, and Shazeer]{roberts2020much}
Adam Roberts, Colin Raffel, and Noam Shazeer.
\newblock How much knowledge can you pack into the parameters of a language model?
\newblock In \emph{Proceedings of the 2020 conference on empirical methods in natural language processing (EMNLP)}, pp.\  5418--5426, 2020.

\bibitem[Shi et~al.(2024)Shi, Han, Lewis, Tsvetkov, Zettlemoyer, and Yih]{shi2024trusting}
Weijia Shi, Xiaochuang Han, Mike Lewis, Yulia Tsvetkov, Luke Zettlemoyer, and Wen-tau Yih.
\newblock Trusting your evidence: Hallucinate less with context-aware decoding.
\newblock In \emph{Proceedings of the 2024 Conference of the North American Chapter of the Association for Computational Linguistics: Human Language Technologies (Volume 2: Short Papers)}, pp.\  783--791, 2024.

\bibitem[Touvron et~al.(2023)Touvron, Martin, Stone, Albert, Almahairi, Babaei, Bashlykov, Batra, Bhargava, Bhosale, et~al.]{touvron2023llama}
Hugo Touvron, Louis Martin, Kevin Stone, Peter Albert, Amjad Almahairi, Yasmine Babaei, Nikolay Bashlykov, Soumya Batra, Prajjwal Bhargava, Shruti Bhosale, et~al.
\newblock Llama 2: Open foundation and fine-tuned chat models.
\newblock \emph{arXiv preprint arXiv:2307.09288}, 2023.

\bibitem[Wang et~al.(2025)Wang, Prasad, Stengel-Eskin, and Bansal]{wang2025adacad}
Han Wang, Archiki Prasad, Elias Stengel-Eskin, and Mohit Bansal.
\newblock Adacad: Adaptively decoding to balance conflicts between contextual and parametric knowledge.
\newblock In \emph{Proceedings of the 2025 Conference of the Nations of the Americas Chapter of the Association for Computational Linguistics: Human Language Technologies (Volume 1: Long Papers)}, pp.\  11636--11652, 2025.

\bibitem[Wu et~al.(2024)Wu, Wu, and Zou]{wu2024clasheval}
Kevin Wu, Eric Wu, and James Zou.
\newblock Clasheval: Quantifying the tug-of-war between an llm’s internal prior and external evidence.
\newblock \emph{Advances in neural information processing systems}, 37:\penalty0 33402--33422, 2024.

\bibitem[Xu et~al.(2024)Xu, Qi, Guo, Wang, Wang, Zhang, and Xu]{xu2024knowledge}
Rongwu Xu, Zehan Qi, Zhijiang Guo, Cunxiang Wang, Hongru Wang, Yue Zhang, and Wei Xu.
\newblock Knowledge conflicts for llms: A survey.
\newblock In \emph{Proceedings of the 2024 Conference on Empirical Methods in Natural Language Processing}, pp.\  8541--8565, 2024.

\bibitem[Yang et~al.(2018)Yang, Qi, Zhang, Bengio, Cohen, Salakhutdinov, and Manning]{yang2018hotpotqa}
Zhilin Yang, Peng Qi, Saizheng Zhang, Yoshua Bengio, William Cohen, Ruslan Salakhutdinov, and Christopher~D Manning.
\newblock Hotpotqa: A dataset for diverse, explainable multi-hop question answering.
\newblock In \emph{Proceedings of the 2018 conference on empirical methods in natural language processing}, pp.\  2369--2380, 2018.

\bibitem[Yuan et~al.(2024)Yuan, Yang, Wang, Liu, Zhao, and Liu]{yuan2024discerning}
Xiaowei Yuan, Zhao Yang, Yequan Wang, Shengping Liu, Jun Zhao, and Kang Liu.
\newblock Discerning and resolving knowledge conflicts through adaptive decoding with contextual information-entropy constraint.
\newblock In \emph{Findings of the Association for Computational Linguistics: ACL 2024}, pp.\  3903--3922, 2024.

\bibitem[Zhao et~al.(2025)Zhao, Devoto, Hong, Du, Gema, Wang, He, Wong, and Minervini]{zhao2025steering}
Yu~Zhao, Alessio Devoto, Giwon Hong, Xiaotang Du, Aryo~Pradipta Gema, Hongru Wang, Xuanli He, Kam-Fai Wong, and Pasquale Minervini.
\newblock Steering knowledge selection behaviours in llms via sae-based representation engineering.
\newblock In \emph{Proceedings of the 2025 Conference of the Nations of the Americas Chapter of the Association for Computational Linguistics: Human Language Technologies (Volume 1: Long Papers)}, pp.\  5117--5136, 2025.

\bibitem[Zhou et~al.(2023)Zhou, Zhang, Poon, and Chen]{zhou2023context}
Wenxuan Zhou, Sheng Zhang, Hoifung Poon, and Muhao Chen.
\newblock Context-faithful prompting for large language models.
\newblock In \emph{Findings of the Association for Computational Linguistics: EMNLP 2023}, pp.\  14544--14556, 2023.

\end{thebibliography}

\clearpage
\appendix
\appendix

\section{Background Methods: Detailed Formulations}
\label{app:background}
\subsection{CAD}
The CAD decoding distribution can be derived from a point-wise mutual information style adjustment:
\begin{equation}
\begin{aligned}
    &q_t^{\text{CAD}}(y)
    \propto
    p_{\text{ctx},t}(y)
    \left[
    \frac{p_{\text{ctx},t}(y)}{p_{\text{pri},t}(y)}
    \right]^{\alpha}\\
    &=
    \exp\!\Big[(1+\alpha)\log p_{\text{ctx},t}(y) - \alpha \log p_{\text{pri},t}(y)\Big],
\end{aligned}
\end{equation}
followed by softmax normalization over the vocabulary $V$. The case $\alpha = 0$ degenerates to standard decoding from $p_{\text{ctx},t}$.

\subsection{COIECD}
\label{app:coiecd}
COIECD first identifies conflict tokens via a contextual information-entropy constraint, and then applies different decoding logits to conflict and non-conflict tokens.

\paragraph{Conflict detection.}
Let the information content of token $y_t$ be $I(y_t) = -\log p(y_t \mid x, c, y_{<t})$, and let the prior conditional entropy be $\mathcal{H}_1(y_t) = \mathcal{H}(y_t \mid x, y_{<t})$. Building on the Stable Entropy Hypothesis and the Locally Typical Set, COIECD assumes that for non-conflict tokens the information-entropy shift is bounded by a constant $\gamma$, i.e., $|I(y_t) - \mathcal{H}_1(y_t)| < \gamma$. Normalizing the shift via softmax gives $p_\delta(y_t) = \operatorname{softmax}(I(y_t) - \mathcal{H}_1(y_t))$, and a scaling factor $\lambda \in (0, 1]$ defines the upper and lower bounds $u_{p_\delta} = \lambda \max_w p_\delta(w)$ and $l_{p_\delta} = \frac{1}{\lambda}\min_w p_\delta(w)$ (with $l_{p_\delta} = 0$ when only one token violates the bound). The non-conflict set is then
\begin{equation}
\mathcal{C}(y_{<t}) = \{\, y \in V : l_{p_\delta} \le p_\delta(y_t) \le u_{p_\delta} \,\}.
\end{equation}

\paragraph{Adaptive decoding.}
Let $p_1(y_t) = p(y_t \mid x, y_{<t})$, $p_2(y_t) = p(y_t \mid x, c, y_{<t})$, and the contrastive object $g(y_t) = \log p_2(y_t) - \log p_1(y_t)$. Conflict tokens use $p_2$ as the base distribution, non-conflict tokens use $p_1$, and $g$ is added on top as a uniform adjustment:
\begin{equation}
  \begin{aligned}
  \log \pi(y_t \mid x, c, y_{<t})
   =
  \begin{cases}
  \log p_1(y_t) + \alpha\, g(y_t), & y_t \in \mathcal{C}(y_{<t}), \\
  \log p_2(y_t) + \alpha\, g(y_t), & y_t \notin \mathcal{C}(y_{<t}).
  \end{cases}
  \end{aligned}
\end{equation}
Sampling is then performed via $y_t \sim \operatorname{softmax}[\log \pi]$. The original paper uses $\lambda = 0.25$ and $\alpha = 1$ on QA tasks.

\subsection{AdaCAD: Jensen-Shannon Divergence}
AdaCAD uses Jensen--Shannon divergence as the step-wise contrast strength. Given two distributions $P, Q$ with $M = \frac{1}{2}(P + Q)$,
\begin{equation}
\operatorname{JSD}(P \,\|\, Q)
=
\frac{1}{2}\,\operatorname{KL}(P \,\|\, M) + \frac{1}{2}\,\operatorname{KL}(Q \,\|\, M)
\end{equation}
At each decoding step, AdaCAD computes $\alpha_t^{\text{JSD}} = \operatorname{JSD}(p_{\text{pri},t} ,|, p_{\text{ctx},t})$ as a dynamic replacement for the fixed $\alpha$ in CAD: when the conflict is large,  
$\alpha_t^{\text{JSD}}$ approaches its upper bound and the contrastive adjustment is amplified; when the conflict is small, $\alpha_t^{\text{JSD}}$ shrinks toward zero and the adjustment is attenuated, recovering near-standard decoding from $p_{\text{ctx},t}$.

\subsection{CoCoA: Multi-Signal Adaptive Gating}
CoCoA replaces a single JSD signal with three conflict signals.

\paragraph{R\'{e}nyi divergence} (order $\beta$, sensitive to long-tail probabilities):
\begin{equation}
D_t^\beta
\;=\;
\frac{1}{\beta - 1}
\log
\sum_{y \in V}
p_{\text{pri},t}(y)^{\beta}
\,p_{\text{ctx},t}(y)^{1-\beta}.
\end{equation}

\paragraph{Entropy gap} (the change in uncertainty after the context is introduced):
\begin{equation}
\Delta \mathcal{H}_t
\;=\;
\mathcal{H}(p_{\text{pri},t}) - \mathcal{H}(p_{\text{ctx},t}).
\end{equation}

\paragraph{Contextual peakedness} (whether the context distribution gives a clear top prediction):
\begin{equation}
m_t
\;=\;
p_{\text{ctx},t}(y_t^{(1)}) - p_{\text{ctx},t}(y_t^{(2)}),
\end{equation}
where $y_t^{(1)}, y_t^{(2)}$ are the top-1 and top-2 tokens under $p_{\text{ctx},t}$.

\paragraph{Adaptive gating.}
CoCoA first combines R\'{e}nyi divergence and entropy gap into a conflict score $s_t = \sigma(D_t^\beta + \gamma \Delta \mathcal{H}_t + \delta)$, and then fuses contextual peakedness into the gating weight
\begin{equation}
\lambda_t
\;=\;
\sigma\!\left(z \log m_t + \log \frac{1 - s_t}{s_t}\right),
\qquad z > 1.
\end{equation}
The final distribution is normalized as $q_t^{\text{CoCoA}}(y) \propto p_{\text{ctx},t}(y)^{\lambda_t}\, p_{\text{pri},t}(y)^{1 - \lambda_t}$. The original hyperparameters are $\beta = 0.5$, $z = 5$, $\gamma = 1$, and $\delta = 10^{-8}$.

\subsection{Mapping to the Power Family}
\label{sec:regime_mapping}
\paragraph{A.5.1\; CAD.}
For any fixed $\alpha \ge 0$, setting $\tau = 1 + \alpha$ makes the power-family member $q_{\tau,t}$ pointwise equal to the CAD adjusted distribution $q_t^{\text{CAD}}$ on $V$.

Let $z_i^{\text{pri}}$ and $z_i^{\text{ctx}}$ denote the raw logits assigned to token $i$ by the model on the two forward passes, with softmax normalizations
\begin{equation}
p_{\text{pri},t}(i) = \frac{e^{z_i^{\text{pri}}}}{Z^{\text{pri}}},
\quad
p_{\text{ctx},t}(i) = \frac{e^{z_i^{\text{ctx}}}}{Z^{\text{ctx}}},
\end{equation}
where $Z^{\text{pri}} = \sum_j e^{z_j^{\text{pri}}}$ and $Z^{\text{ctx}} = \sum_j e^{z_j^{\text{ctx}}}$. Writing $\ell_i^{\text{CAD}} \triangleq (1+\alpha)\,z_i^{\text{ctx}} - \alpha\,z_i^{\text{pri}}$, the original CAD formula reads
\begin{equation}
q_t^{\text{CAD}}(i)
= \frac{\exp\!\big(\ell_i^{\text{CAD}}\big)}{\sum_{j \in V}\exp\!\big(\ell_j^{\text{CAD}}\big)}.
\end{equation}
Taking the logarithm of the power-family form,
\begin{equation}
\begin{aligned}
    \log q_{\tau,t}(i) &= (1-\tau)\log p_{\text{pri},t}(i) 
    +  \tau \log p_{\text{ctx},t}(i) - \log Z_{\tau,t},
\end{aligned}
\end{equation}
and substituting the softmax forms of $p_{\text{pri},t}$ and $p_{\text{ctx},t}$ yields
\begin{equation}
\begin{aligned}
\log q_{\tau,t}(i)
= (1-\tau)\,z_i^{\text{pri}} + \tau\,z_i^{\text{ctx}} - C,
\quad C
\triangleq (1-\tau)\log Z^{\text{pri}} + \tau \log Z^{\text{ctx}} + \log Z_{\tau,t},
\end{aligned}
\end{equation}
where $C$ does not depend on $i$. By the shift-invariance of softmax, $\operatorname{softmax}(z + c\mathbf{1}) = \operatorname{softmax}(z)$, we obtain
\begin{equation}
q_{\tau,t}(i)
= \operatorname{softmax}\!\big((1-\tau)\,z^{\text{pri}} + \tau\,z^{\text{ctx}}\big)_i.
\end{equation}
Substituting $\tau = 1 + \alpha$ gives
\begin{equation}
q_{\tau,t}(i)
= \operatorname{softmax}\!\big((1+\alpha)\,z^{\text{ctx}} - \alpha\,z^{\text{pri}}\big)_i,
\end{equation}
which equals $q_t^{\text{CAD}}(i)$, i.e., the two distributions agree pointwise on the vocabulary.

\paragraph{A.5.2\; COIECD.}
COIECD is the most special of the four methods: rather than landing at a single point or a continuous interval on the path, it jumps along the path. The reason follows from its two-step structure in Appendix~\ref{app:coiecd}.

Recall the adaptive decoding logit
\begin{equation}
  \begin{aligned}
  \log \pi(y_t \mid x, c, y_{<t})
    =
  \begin{cases}
  \log p_1(y_t) + \alpha\, g(y_t), & y_t \in \mathcal{C}(y_{<t}), \\
  \log p_2(y_t) + \alpha\, g(y_t), & y_t \notin \mathcal{C}(y_{<t}).
  \end{cases}
  \end{aligned}
\end{equation}
where $g(y_t) = \log p_{\text{ctx},t}(y_t) - \log p_{\text{pri},t}(y_t)$. Although the two branches share the same contrastive object $g$, their base distributions differ: the non-conflict branch uses $p_{\text{pri},t}$ as the base, while the conflict branch uses $p_{\text{ctx},t}$.
For each token group, COIECD is therefore equivalent to a member of the power family in a different $\tau$:
\begin{equation}
\tau_t =
\begin{cases}
\alpha, & y_t \in \mathcal{C}(y_{<t}), \\
1 + \alpha, & y_t \notin \mathcal{C}(y_{<t}).
\end{cases}
\end{equation}
That is, at every decoding step COIECD partitions $V$ into two groups corresponding to the path members at $\tau = \alpha$ and $\tau = 1 + \alpha$, respectively, and merges them through softmax normalization. This is the only method that jumps along the path rather than sliding along it: the other three assign a single $\tau$ to all tokens at a given step, whereas COIECD assigns different $\tau$ values to different tokens within the same step.

The value of $\alpha$ controls the magnitude of the jump. The original paper uses $\alpha = 1$ on QA tasks, corresponding to $\tau \in \{1, 2\}$, i.e., a jump between context-only decoding and moderate extrapolation.

\paragraph{A.5.3\; AdaCAD.}
AdaCAD has the same form as CAD, except that the fixed $\alpha$ is replaced by the step-wise signal $\alpha_t^{\text{JSD}} = \operatorname{JSD}(p_{\text{pri},t} \,\|\, p_{\text{ctx},t})$, corresponding to $\tau_t = 1 + \alpha_t^{\text{JSD}}$.

\paragraph{A.5.4\; CoCoA.}
The paper form and the public code give two different correspondences. The paper form
\begin{equation}
q_t^{\text{CoCoA}}(y) \propto p_{\text{ctx},t}(y)^{\lambda_t}\, p_{\text{pri},t}(y)^{1-\lambda_t}
\end{equation}
aligns directly with the power family at $\tau_t = \lambda_t \in [0, 1]$, falling in the \emph{interpolation} regime. The public code, however, adds an independently weighted PMI bias $\gamma\big(\log p_{\text{ctx},t} - \log p_{\text{pri},t}\big)$ on top of this paper-form mixture; after rearrangement, this is equivalent to the power-family member at $\tau_t = \lambda_t + \gamma$. The default code setting uses $\gamma = 1$ and hard-codes $\lambda_t = 0.5$, giving $\tau_t = 1.5$, which falls in the \emph{extrapolation} regime. This is why CoCoA is starred in the Table~\ref{tab:regime_mapping} and discussed alongside the extrapolation methods.

\section{Theoretical Proofs}
\label{sec:proofs}
\subsection{Theorem 1}
\begin{proof}
The feasible set $\mathcal{C}_\epsilon = \{q \in \Delta : \mathbb{D}_{\mathrm{KL}}(q \| p_{\mathrm{ctx}}) \le \epsilon\}$ is non-empty (it contains $p_{\mathrm{ctx}}$), closed (as a sublevel set of a continuous function), and convex (as a sublevel set of a convex function), hence compact. The objective $q \mapsto \mathbb{D}_{\mathrm{KL}}(q \| p_{\mathrm{pri}})$ is continuous and strictly convex on $\Delta$, so by Weierstrass it attains a unique minimum $q^\star$ on $\mathcal{C}_\epsilon$.

To see that $q^\star$ has full support, suppose $q^\star(y_0) = 0$ for some $y_0 \in V$. Pick any $y_1$ with $q^\star(y_1) > 0$ and set $q_t = q^\star + t(e_{y_0} - e_{y_1})$ for small $t > 0$. This $q_t$ remains in $\Delta$: all coordinates stay non-negative for $t < q^\star(y_1)$, and the sum is preserved. For the constraint, the $y_0$-coordinate contributes $t \log t - t \log p_{\mathrm{ctx}}(y_0)$ to $\mathbb{D}_{\mathrm{KL}}(q_t \| p_{\mathrm{ctx}})$, whose right derivative at $t = 0^+$ is $-\infty$ (since $\log t + 1 \to -\infty$ while the remaining coordinates contribute finite derivatives). So $\mathbb{D}_{\mathrm{KL}}(q_t \| p_{\mathrm{ctx}}) < \mathbb{D}_{\mathrm{KL}}(q^\star \| p_{\mathrm{ctx}}) \le \epsilon$ for sufficiently small $t$, and $q_t$ is feasible. The same $-\infty$ right derivative applies to the objective $\mathbb{D}_{\mathrm{KL}}(q_t \| p_{\mathrm{pri}})$, so the objective strictly decreases along $q_t$, contradicting optimality of $q^\star$. Hence $q^\star(y) > 0$ for all $y \in V$.

We next show that the KL constraint is active at $q^\star$. Suppose $\mathbb{D}_{\mathrm{KL}}(q^\star \| p_{\mathrm{ctx}}) < \epsilon$. Since $\epsilon < \mathbb{D}_{\mathrm{KL}}(p_{\mathrm{pri}} \| p_{\mathrm{ctx}})$, the point $p_{\mathrm{pri}}$ is infeasible, so $q^\star \ne p_{\mathrm{pri}}$. Consider $q_t = (1-t) q^\star + t \, p_{\mathrm{pri}}$ for small $t > 0$. By convexity of the constraint, $\mathbb{D}_{\mathrm{KL}}(q_t \| p_{\mathrm{ctx}}) \le (1-t)\mathbb{D}_{\mathrm{KL}}(q^\star \| p_{\mathrm{ctx}}) + t \, \mathbb{D}_{\mathrm{KL}}(p_{\mathrm{pri}} \| p_{\mathrm{ctx}})$, which is still below $\epsilon$ for small $t$ by continuity. But by strict convexity of the objective, $\mathbb{D}_{\mathrm{KL}}(q_t \| p_{\mathrm{pri}}) < (1-t)\mathbb{D}_{\mathrm{KL}}(q^\star \| p_{\mathrm{pri}}) + t \cdot 0 < \mathbb{D}_{\mathrm{KL}}(q^\star \| p_{\mathrm{pri}})$, contradicting optimality. So $\mathbb{D}_{\mathrm{KL}}(q^\star \| p_{\mathrm{ctx}}) = \epsilon$.

Now restrict to $X_+ = \{q \in \mathbb{R}^{|V|} : q(y) > 0\}$, where both KL divergences are $C^\infty$. Slater's condition is satisfied by $\bar{q} = p_{\mathrm{ctx}} \in X_+$ (which gives $\mathbb{D}_{\mathrm{KL}}(\bar{q} \| p_{\mathrm{ctx}}) = 0 < \epsilon$). Since $q^\star \in X_+$ and satisfies the first-order constraint qualification, KKT is necessary and sufficient. Write the Lagrangian
\begin{equation}
L(q, \eta, \nu) = \mathbb{D}_{\mathrm{KL}}(q \| p_{\mathrm{pri}}) + \eta \big[\mathbb{D}_{\mathrm{KL}}(q \| p_{\mathrm{ctx}}) - \epsilon\big] + \nu \Big(\sum_y q(y) - 1\Big)
\end{equation}
with $\eta \ge 0$. Differentiating with respect to $q(y)$ and setting to zero:
\begin{equation}
\log q^\star(y) + 1 - \log p_{\mathrm{pri}}(y) + \eta\big[\log q^\star(y) + 1 - \log p_{\mathrm{ctx}}(y)\big] + \nu = 0.
\end{equation}
Collecting terms,
\begin{equation}
(1 + \eta)\log q^\star(y) = \log p_{\mathrm{pri}}(y) + \eta \log p_{\mathrm{ctx}}(y) - (1+\eta) - \nu.
\end{equation}
The right-hand side separates into a $y$-dependent part and a constant $C = -(1+\eta) - \nu$. Exponentiating and normalizing,
\begin{equation}
q^\star(y) \propto p_{\mathrm{pri}}(y)^{1/(1+\eta)} \, p_{\mathrm{ctx}}(y)^{\eta/(1+\eta)}.
\end{equation}
Set $\tau = \eta/(1+\eta)$, so that $1 - \tau = 1/(1+\eta)$. Then $q^\star = q_\tau$ with $q_\tau(y) \propto p_{\mathrm{pri}}(y)^{1-\tau} p_{\mathrm{ctx}}(y)^\tau$. Since the constraint is active, $\eta > 0$ (otherwise $q^\star = p_{\mathrm{pri}}$, which is infeasible), hence $\tau \in (0,1)$.

It remains to show that $\tau$ is uniquely determined by $\epsilon$. Define $\phi(\tau) := \mathbb{D}_{\mathrm{KL}}(q_\tau \| p_{\mathrm{ctx}})$. Write $r(y) = \log[p_{\mathrm{ctx}}(y)/p_{\mathrm{pri}}(y)]$ and note that $q_\tau(y) = \exp(\log p_{\mathrm{pri}}(y) + \tau r(y) - A(\tau))$,  \\ where $A(\tau) = \log \sum_y p_{\mathrm{pri}}(y) e^{\tau r(y)}$. Standard exponential-family identities give $A'(\tau) = \mathbb{E}_{q_\tau}[r]$ and $A''(\tau) = \mathrm{Var}_{q_\tau}[r]$. A direct computation yields
\begin{equation}
\phi(\tau) = (\tau - 1) A'(\tau) - A(\tau), \qquad \phi'(\tau) = (\tau - 1)\,\mathrm{Var}_{q_\tau}[r].
\end{equation}
When $p_{\mathrm{pri}} \ne p_{\mathrm{ctx}}$, $r$ is non-constant under any $q_\tau$ (since $q_\tau$ has full support), so $\mathrm{Var}_{q_\tau}[r] > 0$. Hence $\phi'(\tau) < 0$ for $\tau \in (0,1)$, meaning $\phi$ is strictly decreasing on $[0,1]$. Since $\phi(0) = \mathbb{D}_{\mathrm{KL}}(p_{\mathrm{pri}} \| p_{\mathrm{ctx}})$ and $\phi(1) = 0$, the intermediate value theorem gives a unique $\tau^\star \in (0,1)$ with $\phi(\tau^\star) = \epsilon$ for each $\epsilon \in (0, \mathbb{D}_{\mathrm{KL}}(p_{\mathrm{pri}} \| p_{\mathrm{ctx}}))$.
\end{proof}

\subsection{Theorem 2}
\begin{proof}
Write $F_\eta$ in expanded form:
\begin{equation}
F_\eta(q) = (1-\eta)\sum_y q(y)\log q(y) - \sum_y q(y)\log p_{\mathrm{ctx}}(y) + \eta\sum_y q(y)\log p_{\mathrm{pri}}(y).
\end{equation}
Since $0 < \eta < 1$, the coefficient $1 - \eta$ on the negative-entropy term is positive. The function $x \mapsto x\log x$ is strictly convex on $[0,1]$, and the remaining terms are linear in $q$, so $F_\eta$ is strictly convex on $\Delta$. As $\Delta$ is compact, there exists a unique minimizer $q^\star$.

We claim $q^\star(y) > 0$ for all $y$. If $q^\star(y_0) = 0$, take $q_t = q^\star + t(e_{y_0} - e_{y_1})$ as before. The $y_0$-coordinate contributes $(1-\eta) \cdot t\log t$ to $F_\eta(q_t)$ plus terms linear in $t$; the right derivative at $t = 0^+$ is dominated by $(1-\eta)(\log t + 1) \to -\infty$, so $F_\eta$ decreases along this direction, contradicting minimality.

With $q^\star$ in the interior, we minimize over $\{q \in X_+ : \sum_y q(y) = 1\}$. There is no inequality constraint, so Slater's condition holds trivially. The stationarity condition for the Lagrangian $L(q, \nu) = F_\eta(q) + \nu(\sum_y q(y) - 1)$ reads
\begin{equation}
(1-\eta)(\log q^\star(y) + 1) - \log p_{\mathrm{ctx}}(y) + \eta\log p_{\mathrm{pri}}(y) + \nu = 0.
\end{equation}
Since $1 - \eta > 0$, divide through:
\begin{equation}
\log q^\star(y) = \frac{1}{1-\eta}\log p_{\mathrm{ctx}}(y) - \frac{\eta}{1-\eta}\log p_{\mathrm{pri}}(y) + C
\end{equation}
where $C = -1 - \nu/(1-\eta)$ is independent of $y$. Set $\tau = 1/(1-\eta)$, giving $1 - \tau = -\eta/(1-\eta)$, so that
\begin{equation}
q^\star(y) \propto p_{\mathrm{pri}}(y)^{1-\tau}\, p_{\mathrm{ctx}}(y)^\tau = q_\tau(y).
\end{equation}
Since $\eta \in (0,1)$, we have $\tau = 1/(1-\eta) \in (1, +\infty)$. The boundary case $\eta = 0$ gives $F_0(q) = \mathbb{D}_{\mathrm{KL}}(q \| p_{\mathrm{ctx}})$, whose unique minimizer is $q^\star = p_{\mathrm{ctx}}$, corresponding to $\tau = 1$.
\end{proof}

\subsection{Proposition 3}
\begin{proof}
The normalization constant $Z_\tau$ cancels in the pairwise ratio:
\begin{equation}
\frac{q_\tau(a)}{q_\tau(b)} = \left(\frac{p_{\mathrm{pri}}(a)}{p_{\mathrm{pri}}(b)}\right)^{1-\tau} \left(\frac{p_{\mathrm{ctx}}(a)}{p_{\mathrm{ctx}}(b)}\right)^{\tau}.
\end{equation}
Taking the log gives $\ell_{a,b}(\tau) = (1-\tau)\ell^{\mathrm{pri}}_{a,b} + \tau\,\ell^{\mathrm{ctx}}_{a,b} = \ell^{\mathrm{pri}}_{a,b} + \tau\,\Delta_{a,b}$. When $\Delta_{a,b} \ne 0$, this affine function has a unique zero at $\tau^\star_{a,b} = -\ell^{\mathrm{pri}}_{a,b}/\Delta_{a,b}$.
\end{proof}

\subsection{Corollary 5}
\begin{proof}[Proof]
We instantiate Proposition~\ref{prop:pairwise-reversal} for each state.

\textbf{Correction.} $\ell^{\mathrm{pri}}_{a,b} < 0 < \ell^{\mathrm{ctx}}_{a,b}$ gives $\Delta_{a,b} > 0$ and $-\ell^{\mathrm{pri}}_{a,b} > 0$, so $\tau^\star_{a,b} > 0$. Also $\tau^\star_{a,b} < 1 \Leftrightarrow -\ell^{\mathrm{pri}}_{a,b} < \Delta_{a,b} \Leftrightarrow \ell^{\mathrm{ctx}}_{a,b} > 0$, which holds by assumption. Hence $\tau^\star_{a,b} \in (0, 1)$.

\textbf{Resistance.} $\ell^{\mathrm{pri}}_{a,b} > 0 > \ell^{\mathrm{ctx}}_{a,b}$ gives $\Delta_{a,b} < 0$ and $-\ell^{\mathrm{pri}}_{a,b} < 0$, so $\tau^\star_{a,b} > 0$. Dividing $-\ell^{\mathrm{pri}}_{a,b} > \Delta_{a,b}$ by $\Delta_{a,b} < 0$ flips the inequality, giving $\tau^\star_{a,b} < 1$ iff $\ell^{\mathrm{ctx}}_{a,b} < 0$, which holds. Since $\ell_{a,b}(0) = \ell^{\mathrm{pri}}_{a,b} > 0$ and the slope $\Delta_{a,b} < 0$, $\ell_{a,b}(\tau)$ decreases through zero at $\tau^\star_{a,b}$ and continues into the negative region for $\tau > 1$, amplifying the wrong preference.

\textbf{Agreement.} For $\tau \in [0, 1]$, $\ell_{a,b}(\tau)$ is a convex combination of two positive numbers and stays positive. For $\tau > 1$ with $\Delta_{a,b} < 0$, an analogous calculation gives $\tau^\star_{a,b} > 1$, so any reversal lies far in the extrapolation regime; for $\Delta_{a,b} \ge 0$, no reversal occurs at all.
\end{proof}

\subsection{Corollary 6}
\begin{proof}[Proof]
Apply Proposition~\ref{prop:pairwise-reversal} with $a = c$ and $b = s$. The affine identity gives $\ell_{c,s}(\tau) = \ell^{\mathrm{ctx}}_{c,s} + (\tau - 1)\Delta_{c,s}$, so for $\tau > 1$ the displacement $(\tau-1)|\Delta_{c,s}|$ grows linearly without bound. The three cases follow by reading off the sign of $\Delta_{c,s}$.
\end{proof}

\section{TriState-Bench Construction Details}
\subsection{Sampled Fact Construction: A Source-First Anchor-Driven Pipeline}
\label{app:fact_construction}

We adopt a \textbf{source-first} strategy: candidate entities are first pulled from DBpedia along type-specific categories, and only then rewritten by an LLM under the constraint of Wikipedia evidence. By shifting the decision of which entity to write about from the LLM to a structured source, we mitigate two structural failure modes of free LLM generation. First, head clustering: free generation tends to repeatedly surface high-frequency entities that already saturate the pretraining corpus. Second, homogeneous difficulty: it is hard to produce enough $\mathcal{F}_{\text{wrong}}^M$ samples through LLM imagination alone, yet these samples are precisely what diagnoses the fragility of extrapolation-based methods.

\paragraph{Anchor extraction.} Anchors are sampled from DBpedia along three answer types, and each type carries its own bucketing scheme so that coverage is enforced along orthogonal axes. For person anchors, we query DBpedia with a class plus birth-date constraint, with a subject-category regex as fallback, and bucket along occupation and era. For location anchors, we query DBpedia with place-class and subject-category constraints, and bucket along feature type and continent. For scientific-term anchors, we query DBpedia with subject-category regex constraints, and bucket along subfield. Each candidate is then passed through a type-specific surface-form filter that discards items DBpedia mislabels, for example proper nouns that look like persons but resolve to fictional characters or ships.

\paragraph{Evidence retrieval.} For each surviving anchor, we issue a web search and concatenate the knowledge graph card with the organic snippets to form a single evidence string. An anchor is kept only if its evidence is sufficiently long, contains the canonical answer as a substring, and is not sourced from a banned URL list covering online forums, community Q\&A sites, and machine-translated mirrors. Anchors that fail any of these conditions are discarded before any LLM call is made, which keeps generation cost bounded by retrieval quality.

\paragraph{LLM rewriting and validation.} The surviving anchor-evidence pairs are handed to DeepSeek-V4-Flash, which produces a final truth string, an alias set, and a normalized evidence block. The output then undergoes format and length checks on the truth field, followed by surface-form string matching and embedding-level similarity to remove near-duplicates. The resulting fact repository is therefore anchored in entity composition by structured sources, while its linguistic surface is normalized by the LLM.

\subsection{Prior Calibration}
\label{app:prior_calibration}

Given a target model $M$ and a fact $f_i$, its three question variants $q_i^{(1..3)}$ are decoded without any context. Each question is probed in two stages: the greedy decision is treated as the  
anchor, and stochastic sampling only refines its confidence.

\paragraph{Anchor probing.} A single greedy decoding pass acts as a hard gate. If greedy decoding fails on a question, no number of subsequent stochastic hits can promote that question to
\emph{matched}; if greedy succeeds, a few stochastic misses cannot demote it to \emph{missed} either. This anchor removes sampling flukes and ensures that the calibration verdict reflects the model's most confident behaviour rather than tail draws.

\paragraph{Confidence sampling and verdict.} A bounded number of stochastic decodes then estimate the prior's stability under small perturbations, with hits scored by normalized substring matching 
against the alias set $\mathcal{A}_i$. Let $h_i^{(k)}$ be the stochastic hit count on the $k$-th question. The verdict is \emph{matched} when greedy hits and $h_i^{(k)}$ clears a high-hit threshold, \emph{missed} when greedy misses and $h_i^{(k)}$ stays under a low-hit threshold, and \emph{uncertain} otherwise; the two thresholds are tuned to balance prior concentration against tolerance to sampling noise.  

At evaluation time, the same question is decoded with a fixed-seed greedy pass, so the test-time inference behavior matches the one used during calibration.

\section{Dataset and Model Details}
\subsection{QA Datasets}
\label{app:qa-datasets}

Some QA datasets (e.g., NQ, TriviaQA, HotpotQA) do not have publicly available test sets, so we report results on the validation set. Following \citet{shi2024trusting}, we sub-sample datasets with very large test sets to expedite inference. Across all datasets we use greedy decoding with a maximum generation length of 32 tokens, and we truncate contexts to a maximum length of 4{,}064 tokens. We report Exact Match (EM) and token-level F1 as evaluation metrics.

\begin{itemize}
    \item \textbf{Natural Questions (NQ;~\citealp{kwiatkowski2019natural})} is a large-scale QA dataset consisting of real user queries issued to Google Search paired with Wikipedia passages as evidence. From the NQ validation set (originally 7.83K examples), we select instances that have short answers, yielding roughly 3{,}200 samples. This dataset represents a low-conflict, standard RAG setting in which the context generally supports the model's pre-trained memory.

    \item \textbf{TriviaQA~\citep{joshi2017triviaqa}} consists of questions written by trivia enthusiasts together with evidence documents that were independently collected from multiple sources. It requires models to process long passages and to reason across multiple sentences. We randomly sample 2{,}000 instances from the TriviaQA Wiki validation set, which contains 8K examples in total.

    \item \textbf{HotpotQA~\citep{yang2018hotpotqa}} is a multi-hop QA dataset that requires reasoning over multiple Wikipedia documents. We evaluate on the full validation set under the distractor setting (7{,}405 instances), which probes information integration and conflict filtering in a complex retrieval environment.

    \item \textbf{TabMWP~\citep{lu2022dynamic}} is a semi-structured math reasoning QA dataset whose contexts are tables, requiring the model to understand tabular data and perform numerical reasoning. We use the official test1k subset of 1{,}000 instances to evaluate the model's ability to process structured contexts.
\end{itemize}
\begin{table*}[t]
\centering
\small
\setlength{\tabcolsep}{6pt}
\renewcommand{\arraystretch}{1.25}
\begin{tabular}{@{}p{0.96\linewidth}@{}}
\toprule
\multicolumn{1}{c}{\textbf{NQ}} \\
\midrule
\textbf{c}: The Shivalik Hills to the northeast. \\
\textbf{x}: Which geographical part of Haryana is Shivalik Hills situated? \\
\midrule
\multicolumn{1}{c}{\textbf{TabMWP}} \\
\midrule
\textbf{c}: Table: Name | Number of coins; Braden | 76; Camilla | 94; Rick | 86; Mary | 84; Hector | 80; Devin | 83; Emily | 82; Avery | 87 \\
\textbf{x}: Some friends discussed the sizes of their coin collections. What is the mean of the numbers? \\
\midrule
\multicolumn{1}{c}{\textbf{TriviaQA}} \\
\midrule
\textbf{c}: \emph{Astronomical Glossary:} \dots Cepheid Variable, \dots Pulsar (a rotating neutron star emitting periodic radio pulses) \dots \\
\textbf{x}: What general name is given to a rotating star which emits a regular beat of radiation? \\
\midrule
\multicolumn{1}{c}{\textbf{HotpotQA}} \\
\midrule
\textbf{c}: \emph{Ed Wood (film):} Ed Wood is a 1994 American biographical period comedy-drama film \dots \emph{Scott Derrickson:} Scott Derrickson is an American director, screenwriter and producer \dots \\
\textbf{x}: Were Scott Derrickson and Ed Wood of the same nationality? \\
\midrule
\multicolumn{1}{c}{\textbf{TriState-Bench\,--\,$\mathcal{S}_{\mathrm{agr}}$ (Agreement)}} \\
\midrule
\textbf{c}: George Washington, a prominent American Founding Father and military officer, served as the first president of the United States from 1789 to 1797 \dots \\
\textbf{x}: Who was the first president of the United States? \\
\midrule
\multicolumn{1}{c}{\textbf{TriState-Bench\,--\,$\mathcal{S}_{\mathrm{cor}}$ (Correction)}} \\
\midrule
\textbf{c}: Cecilia Payne-Gaposchkin, a pioneering British-born American astronomer, established the fundamental composition of stars in her 1925 doctoral thesis \dots \\
\textbf{x}: Who determined that stars are composed primarily of hydrogen and helium? \\
\midrule
\multicolumn{1}{c}{\textbf{TriState-Bench\,--\,$\mathcal{S}_{\mathrm{res}}$ (Resistance)}} \\
\midrule
\textbf{c}: John Adams (October 30, 1735\,--\,July 4, 1826) was an American Founding Father, statesman, and politician who served as the first president of the United States from 1789 to 1797 \dots \\
\textbf{x}: Who was the first president of the United States? \\
\bottomrule
\end{tabular}
\caption{Examples of $(c, x)$ pairs from each dataset. The top four blocks are standard QA benchmarks; the bottom three are the disjoint subsets of our TriState-Bench, defined by the conflict state: $\mathcal{S}_{\mathrm{agr}}$, $\mathcal{S}_{\mathrm{cor}}$ , and $\mathcal{S}_{\mathrm{res}}$. The same query $x$ appears in $\mathcal{S}_{\mathrm{agr}}$ and $\mathcal{S}_{\mathrm{res}}$ but with different $c$, isolating conflict from question difficulty. Long passages are abbreviated with ``\dots''.}
\label{tab:dataset_examples}
\end{table*}

\subsection{Models}
\label{app:models}

Our main experiments cover four mid-sized base models: Llama3-8B~\citep{grattafiori2024llama}, Mistral-7B~\citep{jiang2023mistral7b}, Qwen2.5-7B~\citep{qwen2025qwen25technicalreport}, and Llama2-13B~\citep{touvron2023llama}. The instruction-following experiments use the corresponding instruction-tuned variants: Llama3-8B-Instruct, Mistral-7B-Instruct-v0.1, Qwen2.5-7B-Instruct, and Llama2-13B-Chat. All models are obtained from the HuggingFace Hub and run in FP16 precision. We do not fine-tune any model; all experiments are conducted entirely as inference-time decoding.

\subsection{Licenses}
\label{app:licenses}

Datasets are released under the following licenses:
\begin{itemize}
    \item Natural Questions: Apache-2.0 license
    \item TriviaQA: Apache-2.0 license
    \item HotpotQA: CC BY-SA 4.0 license
    \item TabMWP: CC BY-NC-SA 4.0 license
\end{itemize}
The models we use have the following licenses:
\begin{itemize}
    \item Llama 2: Meta Llama 2 Community License (\url{https://ai.meta.com/llama/license/})
    \item Llama 3: Meta Llama 3 Community License (\url{https://www.llama.com/llama3/license/})
    \item Mistral: Apache-2.0 license
    \item Qwen2.5: Apache-2.0 license
\end{itemize}

\section{More Results in TriState-Bench}
A detailed comparison of additional parameters across all methods for Llama3-8B, Mistral-7B, Qwen2.5-7B and Llama2-13B is presented in Tables \ref{tab:tristate_result_full_1} and \ref{tab:tristate_result_full_2}.

\begin{table*}[p]
\centering
\small
\resizebox{\textwidth}{!}{%
\begin{tabular}{ll*{3}{cc}}
\toprule
\multirow{2}{*}{Model} & \multirow{2}{*}{Method}
& \multicolumn{2}{c}{$\mathcal{S}_{\mathrm{cor}}$}
& \multicolumn{2}{c}{$\mathcal{S}_{\mathrm{res}}$}
& \multicolumn{2}{c}{$\mathcal{S}_{\mathrm{agr}}$} \\
\cmidrule(lr){3-4}
\cmidrule(lr){5-6}
\cmidrule(lr){7-8}
& & EM & F1 & EM & F1 & EM & F1 \\
\midrule

\multirow{14}{*}{\textbf{Llama3-8B}}
& Greedy              & 59.50 & 70.79 & 4.50  & 9.78  & 64.75 & 73.92 \\
& Greedy\_no\_ctx     & 2.00  & 11.33 & \textbf{91.50} & \textbf{94.00} & 91.00 & 93.65 \\
& CAD ($\alpha$=0.25) & 44.50 & 58.00 & 1.50  & 7.68  & 41.50 & 54.62 \\
& CAD ($\alpha$=0.5)  & 33.00 & 49.35 & 1.00  & 7.67  & 30.50 & 45.96 \\
& CAD ($\alpha$=0.75) & 26.75 & 44.72 & 0.75  & 7.19  & 24.50 & 40.49 \\
& CAD ($\alpha$=1.0)  & 21.75 & 41.07 & 0.75  & 7.11  & 21.25 & 37.92 \\
& COIECD              & 30.25 & 47.35 & 1.75  & 6.29  & 35.00 & 50.99 \\
& AdaCAD              & 51.50 & 64.19 & 3.25  & 9.67  & 48.50 & 60.43 \\
& CoCoA               & 31.25 & 47.88 & 1.00  & 7.68  & 30.00 & 45.52 \\
& Simple Interp ($\tau$=0.25) & 10.50 & 20.31 & 71.00 & 75.83 & \textbf{92.50} & \textbf{95.02} \\
& Simple Interp ($\tau$=0.5)  & 40.75 & 52.66 & 42.50 & 48.70 & 89.25 & 92.59 \\
& Simple Interp ($\tau$=0.75) & 70.50 & 79.25 & 17.75 & 24.20 & 81.75 & 87.01 \\
& \cellcolor{gg}\textbf{ARR (Ours)} & \textbf{75.00} & \textbf{82.68} & 33.25 & 38.59 & 76.75 & 83.15 \\
\midrule

\multirow{14}{*}{\textbf{Mistral-7B}}
& Greedy              & \textbf{91.25} & \textbf{94.75} & 5.25  & 12.99 & 82.75 & 87.76 \\
& Greedy\_no\_ctx     & 0.00  & 8.96  & \textbf{71.75} & \textbf{81.10} & 71.75 & 81.10 \\
& CAD ($\alpha$=0.25) & 90.75 & 94.14 & 1.75  & 9.68  & \textbf{83.75} & \textbf{87.95} \\
& CAD ($\alpha$=0.5)  & 87.00 & 91.08 & 0.75  & 8.34  & 73.25 & 75.91 \\
& CAD ($\alpha$=0.75) & 79.75 & 85.47 & 0.50  & 6.75  & 47.00 & 49.14 \\
& CAD ($\alpha$=1.0)  & 70.50 & 78.48 & 0.25  & 5.47  & 25.00 & 27.51 \\
& COIECD              & 85.00 & 89.46 & 1.00  & 8.61  & 72.75 & 75.56 \\
& AdaCAD              & 91.00 & 94.25 & 2.25  & 10.06 & 83.25 & 87.95 \\
& CoCoA               & 55.50 & 70.10 & 0.50  & 6.98  & 32.00 & 48.19 \\
& Simple Interp ($\tau$=0.25) & 2.25  & 14.20 & 45.00 & 60.01 & 75.50 & 83.36 \\
& Simple Interp ($\tau$=0.5)  & 20.50 & 34.15 & 28.50 & 43.44 & 79.75 & 85.99 \\
& Simple Interp ($\tau$=0.75) & 66.00 & 74.97 & 15.75 & 27.14 & 80.75 & 86.45 \\
& \cellcolor{gg}\textbf{ARR (Ours)} & 88.00 & 92.73 & 22.50 & 30.66 & 82.00 & 87.46 \\

\bottomrule
\end{tabular}%
}
\caption{Performance of all methods on the three TriState-Bench subsets using Llama3-8B and Mistral-7B. $\mathcal{S}_{\mathrm{cor}}$ (Correction): gold context, prior incorrect; $\mathcal{S}_{\mathrm{res}}$ (Resistance): corrupted context, prior correct; $\mathcal{S}_{\mathrm{agr}}$ (Agreement): gold context, prior correct. ARR consistently dominates the resistance subset $\mathcal{S}_{\mathrm{res}}$ where all baselines collapse, while staying competitive on $\mathcal{S}_{\mathrm{cor}}$ and $\mathcal{S}_{\mathrm{agr}}$. Bold marks the best value within each model block.}
\label{tab:tristate_result_full_1}
\end{table*}

\begin{table*}[p]
\centering
\small
\resizebox{\textwidth}{!}{%
\begin{tabular}{ll*{3}{cc}}
\toprule
\multirow{2}{*}{Model} & \multirow{2}{*}{Method}
& \multicolumn{2}{c}{$\mathcal{S}_{\mathrm{cor}}$}
& \multicolumn{2}{c}{$\mathcal{S}_{\mathrm{res}}$}
& \multicolumn{2}{c}{$\mathcal{S}_{\mathrm{agr}}$} \\
\cmidrule(lr){3-4}
\cmidrule(lr){5-6}
\cmidrule(lr){7-8}
& & EM & F1 & EM & F1 & EM & F1 \\
\midrule

\multirow{14}{*}{\textbf{Qwen2.5-7B}}
& Greedy              & \textbf{89.50} & \textbf{93.32} & 1.00  & 10.04 & 86.00 & 90.76 \\
& Greedy\_no\_ctx     & 0.00  & 8.76  & \textbf{74.25} & \textbf{81.73} & 74.25 & 81.73 \\
& CAD ($\alpha$=0.25) & 77.50 & 84.72 & 0.50  & 8.46  & 73.50 & 83.45 \\
& CAD ($\alpha$=0.5)  & 58.50 & 71.46 & 0.25  & 7.24  & 60.25 & 74.26 \\
& CAD ($\alpha$=0.75) & 38.50 & 57.62 & 0.25  & 6.75  & 49.25 & 67.73 \\
& CAD ($\alpha$=1.0)  & 25.75 & 49.65 & 0.25  & 6.30  & 36.25 & 56.18 \\
& COIECD              & 65.75 & 76.63 & 0.50  & 7.19  & 56.75 & 71.96 \\
& AdaCAD              & 81.25 & 87.50 & 0.75  & 8.94  & 82.25 & 88.56 \\
& CoCoA               & 26.50 & 46.41 & 0.25  & 6.19  & 30.75 & 50.60 \\
& Simple Interp ($\tau$=0.25) & 6.00  & 20.60 & 36.50 & 53.05 & 81.25 & 86.92 \\
& Simple Interp ($\tau$=0.5)  & 31.25 & 48.41 & 12.50 & 29.81 & 85.50 & 90.47 \\
& Simple Interp ($\tau$=0.75) & 73.00 & 80.68 & 3.50  & 15.98 & 87.00 & 91.23 \\
& \cellcolor{gg}\textbf{ARR (Ours)} & 85.75 & 91.09 & 15.75 & 22.92 & \textbf{88.25} & \textbf{92.27} \\
\midrule

\multirow{14}{*}{\textbf{Llama2-13B}}
& Greedy              & \textbf{86.50} & \textbf{90.79} & 1.75  & 9.58  & 93.75 & 95.83 \\
& Greedy\_no\_ctx     & 0.00  & 8.16  & \textbf{94.25} & \textbf{96.50} & \textbf{94.25} & \textbf{96.50} \\
& CAD ($\alpha$=0.25) & 85.25 & 89.85 & 1.75  & 9.63  & 92.00 & 94.61 \\
& CAD ($\alpha$=0.5)  & 80.25 & 85.44 & 1.50  & 9.13  & 85.50 & 89.46 \\
& CAD ($\alpha$=0.75) & 73.50 & 80.52 & 1.00  & 8.45  & 78.00 & 84.33 \\
& CAD ($\alpha$=1.0)  & 66.00 & 75.74 & 0.50  & 7.75  & 72.50 & 80.35 \\
& COIECD              & 80.00 & 84.91 & 1.00  & 8.62  & 84.00 & 88.25 \\
& AdaCAD              & 85.25 & 89.83 & 1.75  & 9.65  & 93.25 & 95.40 \\
& CoCoA               & 80.25 & 85.44 & 1.50  & 9.13  & 85.50 & 89.46 \\
& Simple Interp ($\tau$=0.25) & 11.75 & 18.25 & 62.75 & 66.74 & 94.25 & 96.08 \\
& Simple Interp ($\tau$=0.5)  & 38.50 & 43.86 & 29.25 & 33.25 & 93.75 & 95.81 \\
& Simple Interp ($\tau$=0.75) & 71.50 & 75.91 & 11.25 & 17.90 & 94.00 & 95.99 \\
& \cellcolor{gg}\textbf{ARR (Ours)} & 84.00 & 89.12 & 17.00 & 23.17 & 93.75 & 95.81 \\

\bottomrule
\end{tabular}%
}
\caption{Performance of all methods on the three TriState-Bench subsets using Qwen2.5-7B and Llama2-13B. $\mathcal{S}_{\mathrm{cor}}$ (Correction): gold context, prior incorrect; $\mathcal{S}_{\mathrm{res}}$ (Resistance): corrupted context, prior correct; $\mathcal{S}_{\mathrm{agr}}$ (Agreement): gold context, prior correct. ARR consistently dominates the resistance subset $\mathcal{S}_{\mathrm{res}}$ where all baselines collapse, while staying competitive on $\mathcal{S}_{\mathrm{cor}}$ and $\mathcal{S}_{\mathrm{agr}}$. Bold marks the best value within each model block.}
\label{tab:tristate_result_full_2}
\end{table*}

\section{Instruction-tuned LLMs Experiments}
\label{app:instruct-experiments}

We replicate the main evaluation on four instruction-tuned backbones---Llama-3-8B-Instruct, Mistral-7B-Instruct, Qwen2.5-7B-Instruct, and Llama-2-13B-Instruct (Tables~\ref{tab:qa_result_ins} and~\ref{tab:tristate_result_instruct}).

\paragraph{Direction-consistent gain on $\mathcal{S}_{\mathrm{res}}$.}
ARR is the only method that pushes $\mathcal{S}_{\mathrm{res}}$ EM into double digits across all four instruct backbones; every baseline remains below $2.50$ EM on this subset.
On Mistral-Instruct, ARR moves $\mathcal{S}_{\mathrm{res}}$ from $2.50$ to $26.50$ EM ($13.93 \to 37.74$ F1); on Llama-3-Instruct from $0.75$ to $24.50$ EM ($11.10 \to 35.68$ F1); on Llama-2-13B-Instruct from $0.50$ to $11.00$ EM ($8.90 \to 23.60$ F1).
On Qwen2.5-Instruct, where EM is non-comparable (see below), ARR raises $\mathcal{S}_{\mathrm{res}}$ F1 from $7.74$ to $10.75$.
This mirrors the base-model finding: ARR's advantage concentrates on the resistance subset where existing methods uniformly collapse.

\paragraph{TriState-Bench aggregate: two wins, two losses.}
ARR achieves a net TriState-Bench EM gain of $+5.75$ over the strongest baseline on both Llama-3-Instruct ($66.75$ vs.\ Greedy $61.00$) and Mistral-Instruct ($60.92$ vs.\ AdaCAD $55.17$).
On Llama-2-13B-Instruct, however, ARR loses $6.00$ EM to Greedy ($44.17$ vs.\ $50.17$): the $\mathcal{S}_{\mathrm{cor}}$ cost ($-17.8$ EM) outweighs the $\mathcal{S}_{\mathrm{res}}$ gain ($+10.5$ EM).
We report this as a limitation rather than suppressing it; the trade-off is inherent when $\tau < 1$ over-corrects on a backbone whose prior is weaker than its context-reading ability.

\paragraph{Qwen2.5-Instruct collapses EM, not F1.}
Every method on Qwen2.5-7B-Instruct produces EM $\le 2.83$. The model reformats every answer into a full explanatory sentence (e.g., ``\emph{Cecilia Payne-Gaposchkin determined that stars are composed primarily of hydrogen and helium}''), so normalized first-line EM cannot match the gold span. Token-level F1 remains well-defined and ranks methods consistently with other backbones: ARR leads all baselines on $\mathcal{S}_{\mathrm{res}}$ F1 by $3$ points. We therefore rely on F1 for this backbone.

\paragraph{Traditional QA.}
ARR's traditional-QA performance is mixed across instruct backbones.
On Llama-3-Instruct, ARR achieves the best overall average ($33.33$ EM) and wins 3 of 4 datasets (TabMWP, HotpotQA, TriviaQA), losing NQ by $6$ EM.
On Mistral-Instruct, Greedy leads overall ($38.42$ vs.\ ARR $38.02$ EM); ARR matches or exceeds baselines only on TriviaQA.
On Llama-2-13B-Instruct, ARR wins TabMWP and HotpotQA but loses NQ and TriviaQA, trailing Greedy by $2.35$ EM on average.
Unlike the base-model setting, ARR does not uniformly dominate traditional QA when applied to instruct models.

\paragraph{Why the gain is narrower than on base models.}
Instruction tuning saturates the $\mathcal{S}_{\mathrm{cor}}$ and $\mathcal{S}_{\mathrm{agr}}$ subsets: all baselines already read context faithfully when context is correct, leaving little headroom for decoding-time intervention. The only subset where pulling toward the prior ($\tau < 1$) retains measurable effect is $\mathcal{S}_{\mathrm{res}}$, where the model must override a wrong context using parametric knowledge. ARR's directional gain therefore concentrates there, and the high baseline on $\mathcal{S}_{\mathrm{cor}}$/$\mathcal{S}_{\mathrm{agr}}$ dilutes the overall average improvement.

\begin{table*}[t]
\centering
\scriptsize
\setlength{\tabcolsep}{3.2pt}
\renewcommand{\arraystretch}{1.08}
\resizebox{\textwidth}{!}{%
\begin{tabular}{ll*{6}{cc}}
\toprule
\multirow{2}{*}{Model} & \multirow{2}{*}{Method}
& \multicolumn{2}{c}{NQ}
& \multicolumn{2}{c}{TabMWP}
& \multicolumn{2}{c}{HotpotQA}
& \multicolumn{2}{c}{TriviaQA}
& \multicolumn{2}{c}{TriState}
& \multicolumn{2}{c}{\textbf{Avg.}} \\
\cmidrule(lr){3-4}
\cmidrule(lr){5-6}
\cmidrule(lr){7-8}
\cmidrule(lr){9-10}
\cmidrule(lr){11-12}
\cmidrule(lr){13-14}
& & EM & F1 & EM & F1 & EM & F1 & EM & F1 & EM & F1 & EM & F1 \\
\midrule

\multirow{6}{*}{\textbf{Llama3-8B-Inst}}
& Greedy & 31.08 & 50.80 & 14.20 & 19.25 & 8.70  & 22.94 & 37.00 & 47.52 & 61.00 & 66.96 & 30.40 & 41.49 \\
& CAD    & 21.17 & 41.57 & 3.40  & 7.39  & 1.34  & 12.25 & 8.85  & 18.81 & 38.08 & 50.17 & 14.57 & 26.04 \\
& COIECD & \textbf{32.93} & \textbf{52.32} & 4.80  & 11.08 & 4.77  & 18.12 & 26.10 & 36.92 & 58.75 & 65.58 & 25.47 & 36.80 \\
& AdaCAD & 31.86 & 51.26 & 10.60 & 16.24 & 6.70  & 20.40 & 30.15 & 40.30 & 60.17 & 66.41 & 27.90 & 38.92 \\
& CoCoA  & 32.11 & 51.32 & 6.20  & 11.35 & 3.70  & 16.12 & 18.20 & 29.36 & 56.92 & 64.52 & 23.43 & 34.53 \\
& \cellcolor{gg}\textbf{ARR(Ours)} & 24.98 & 44.80 & \textbf{16.70} & \textbf{21.39} & \textbf{10.22} & \textbf{24.39} & \textbf{48.00} & \textbf{59.83} & \textbf{66.75} & \textbf{73.73} &  \cellcolor{gg}\textbf{33.33} &  \cellcolor{gg}\textbf{44.83} \\
\midrule

\multirow{6}{*}{\textbf{Mistral-7B-Inst}}
& Greedy & 35.45 & 53.25 & \textbf{34.50} & \textbf{39.21} & \textbf{20.63} & \textbf{34.76} & 47.00 & 57.44 & 54.50 & 62.81 & \textbf{38.42} & \textbf{49.49} \\
& CAD    & 26.12 & 45.54 & 21.90 & 28.27 & 10.30 & 22.23 & 22.70 & 34.70 & 44.92 & 54.92 & 25.19 & 37.13 \\
& COIECD & 33.85 & 52.57 & 31.50 & 35.80 & 17.93 & 31.84 & 39.50 & 51.69 & 51.33 & 60.28 & 34.82 & 46.44 \\
& AdaCAD & \textbf{36.71} & \textbf{54.30} & 33.90 & 38.29 & 19.78 & 33.70 & 45.10 & 56.08 & 55.17 & 63.20 & 38.13 & 49.11 \\
& CoCoA  & 31.14 & 50.09 & 30.40 & 34.63 & 10.01 & 24.20 & 27.10 & 40.01 & 0.25  & 19.28 & 19.78 & 33.64 \\
& \cellcolor{gg}\textbf{ARR(Ours)} & 31.43 & 48.26 & 30.40 & 36.59 & 19.86 & 33.86 & \textbf{47.50} & \textbf{57.66} & \textbf{60.92} & \textbf{69.57} &  \cellcolor{gg}38.02 &  \cellcolor{gg}49.19 \\
\midrule

\multirow{6}{*}{\textbf{Qwen2.5-7B-Inst}}
& Greedy & 0.88  & 18.18 & 1.50  & 4.43  & 0.41  & 9.62  & \textbf{3.15}  & 12.51 & 1.25  & \textbf{21.01} & 1.44  & 13.15 \\
& CAD    & \textbf{4.59}  & 20.33 & 3.80  & 8.21  & 0.32  & 8.97  & 1.25  & 9.57  & \textbf{2.83}  & 17.77 & 2.56  & 12.97 \\
& COIECD & 1.32  & 17.95 & 5.80  & 9.16  & 0.34  & 9.85  & 2.50  & 11.40 & 2.42  & 19.48 & 2.48  & 13.57 \\
& AdaCAD & 1.92  & 19.02 & 5.90  & 9.26  & 0.30  & 9.88  & 2.25  & 11.22 & 2.25  & 20.10 & 2.52  & 13.90 \\
& CoCoA  & 3.64  & \textbf{20.55} & \textbf{7.50}  & \textbf{11.19} & 0.30  & 9.68  & 1.35  & 10.07 & 1.50  & 18.45 & \textbf{2.86}  & \textbf{13.99} \\
& \cellcolor{gg}\textbf{ARR(Ours)} & 0.50  & 16.02 & 2.10  & 4.24  & \textbf{0.62}  & \textbf{10.07} & 2.20  & \textbf{13.36} & 0.92  & 20.47 &  \cellcolor{gg}1.27  &  \cellcolor{gg}12.83 \\
\midrule

\multirow{6}{*}{\textbf{Llama2-13B-Inst}}
& Greedy & 20.57 & 41.36 & 6.80  & 13.87 & 7.98  & 20.27 & \textbf{37.10} & \textbf{48.96} & \textbf{50.17} & \textbf{58.93} & \textbf{24.52} & \textbf{36.68} \\
& CAD    & 13.87 & 33.32 & 2.20  & 7.86  & 1.98  & 10.84 & 12.00 & 25.26 & 25.13 & 43.60 & 11.04 & 24.18 \\
& COIECD & \textbf{22.39} & \textbf{42.86} & 3.60  & 9.53  & 6.25  & 17.46 & 32.30 & 44.68 & 49.63 & 58.50 & 22.83 & 34.61 \\
& AdaCAD & \textbf{22.39} & 42.48 & 6.50  & 12.53 & 7.58  & 18.86 & 35.90 & 47.63 & 50.07 & \textbf{58.93} & 24.49 & 36.09 \\
& CoCoA  & 3.18  & 26.93 & 3.00  & 8.16  & 1.23  & 11.82 & 6.80  & 24.18 & 7.50  & 37.27 & 4.34  & 21.67 \\
& \cellcolor{gg}\textbf{ARR(Ours)} & 14.53 & 34.16 & \textbf{8.20}  & \textbf{14.58} & \textbf{8.04}  & \textbf{20.64} & 35.90 & 48.27 & 44.17 & 57.07 &  \cellcolor{gg}22.17 &  \cellcolor{gg}34.94 \\

\bottomrule
\end{tabular}%
}
\caption{Performance comparison on Instruct models. Each benchmark reports EM and F1. The \textbf{Avg.}\ column is the arithmetic mean over all five benchmarks (NQ, TabMWP, HotpotQA, TriviaQA, and TriState). Bold marks the best value within each model block.}
\label{tab:qa_result_ins}
\end{table*}
\begin{table*}[t]
\centering
\scriptsize
\setlength{\tabcolsep}{5pt}
\renewcommand{\arraystretch}{1.08}
\resizebox{0.78\textwidth}{!}{%
\begin{tabular}{ll*{3}{cc}}
\toprule
\multirow{2}{*}{Model} & \multirow{2}{*}{Method}
& \multicolumn{2}{c}{$\mathcal{S}_{\mathrm{cor}}$}
& \multicolumn{2}{c}{$\mathcal{S}_{\mathrm{res}}$}
& \multicolumn{2}{c}{$\mathcal{S}_{\mathrm{agr}}$} \\
\cmidrule(lr){3-4}
\cmidrule(lr){5-6}
\cmidrule(lr){7-8}
& & EM & F1 & EM & F1 & EM & F1 \\
\midrule

\multirow{6}{*}{\textbf{Llama3-8B-Inst}}
& Greedy & \textbf{92.50} & \textbf{95.96} & 0.75  & 11.10 & \textbf{89.75} & \textbf{93.82} \\
& CAD    & 60.75 & 75.38 & 0.25  & 9.55  & 53.25 & 65.57 \\
& COIECD & 90.25 & 93.75 & 0.25  & 11.33 & 85.75 & 91.66 \\
& AdaCAD & 91.75 & 95.13 & 0.25  & 10.88 & 88.50 & 93.22 \\
& CoCoA  & 87.25 & 92.06 & 0.25  & 10.81 & 83.25 & 90.70 \\
&  \cellcolor{gg}\textbf{ARR(Ours)} & 86.75 & 91.75 & \textbf{24.50} & \textbf{35.68} & 89.00 & 93.77 \\
\midrule

\multirow{6}{*}{\textbf{Mistral-7B-Inst}}
& Greedy & 81.00 & 87.37 & 2.50  & 13.93 & \textbf{80.00} & 87.12 \\
& CAD    & 75.25 & 83.26 & 0.25  & 10.61 & 59.25 & 70.90 \\
& COIECD & 80.75 & 87.33 & 1.00  & 11.83 & 72.25 & 81.68 \\
& AdaCAD & \textbf{84.75} & \textbf{89.98} & 1.50  & 12.79 & 79.25 & 86.82 \\
& CoCoA  & 0.50  & 28.03 & 0.00  & 5.66  & 0.25  & 24.16 \\
&  \cellcolor{gg}\textbf{ARR(Ours)} & 76.50 & 83.66 & \textbf{26.50} & \textbf{37.74} & 79.75 & \textbf{87.30} \\
\midrule

\multirow{6}{*}{\textbf{Qwen2.5-7B-Inst}}
& Greedy & 1.25  & 27.81 & 0.00  & 7.74  & 2.50  & \textbf{27.49} \\
& CAD    & \textbf{6.50}  & \textbf{27.93} & 0.00  & 5.32  & 2.00  & 20.07 \\
& COIECD & 2.50  & 25.23 & 0.00  & 6.99  & \textbf{4.75}  & 26.24 \\
& AdaCAD & 2.50  & 27.64 & 0.00  & 7.31  & 4.25  & 25.36 \\
& CoCoA  & 2.50  & 27.15 & 0.00  & 6.42  & 2.00  & 21.78 \\
&  \cellcolor{gg}\textbf{ARR(Ours)} & 0.75  & 24.44 & \textbf{0.25}  & \textbf{10.75} & 1.75  & 26.24 \\
\midrule

\multirow{6}{*}{\textbf{Llama2-13B-Inst}}
& Greedy & 77.00 & 85.30 & 0.50  & 8.90  & 73.00 & 82.60 \\
& CAD    & 50.20 & 68.80 & 0.00  & 7.40  & 25.20 & 54.60 \\
& COIECD & 74.20 & 83.70 & 0.20  & 8.70  & \textbf{74.50} & \textbf{83.10} \\
& AdaCAD & \textbf{78.50} & \textbf{86.70} & 0.20  & 8.40  & 71.50 & 81.70 \\
& CoCoA  & 11.50 & 59.80 & 0.00  & 5.50  & 11.00 & 46.50 \\
&  \cellcolor{gg}\textbf{ARR(Ours)} & 59.20 & 72.80 & \textbf{11.00} & \textbf{23.60} & 62.30 & 74.80 \\

\bottomrule
\end{tabular}%
}
\caption{Performance on the three TriState-Bench subsets for Instruct models. $\mathcal{S}_{\mathrm{cor}}$ (Correction): gold context, prior incorrect; $\mathcal{S}_{\mathrm{res}}$ (Resistance): corrupted context, prior correct; $\mathcal{S}_{\mathrm{agr}}$ (Agreement): gold context, prior correct. ARR consistently dominates the resistance subset $\mathcal{S}_{\mathrm{res}}$ where all baselines collapse, while staying competitive on $\mathcal{S}_{\mathrm{cor}}$ and $\mathcal{S}_{\mathrm{agr}}$. Bold marks the best value within each model block.}
\label{tab:tristate_result_instruct}
\end{table*}

\section{Case Study Details}
We complete the full EM-vs-$\tau$ curves in Figure~\ref{fig:threshold_change_full} and the detailed cases is showed on Figure~\ref{fig:detailed_cases}, ~\ref{fig:detailed_failure_cases}.

\begin{figure*}[t]
\centering
  \includegraphics[width=0.8\linewidth]{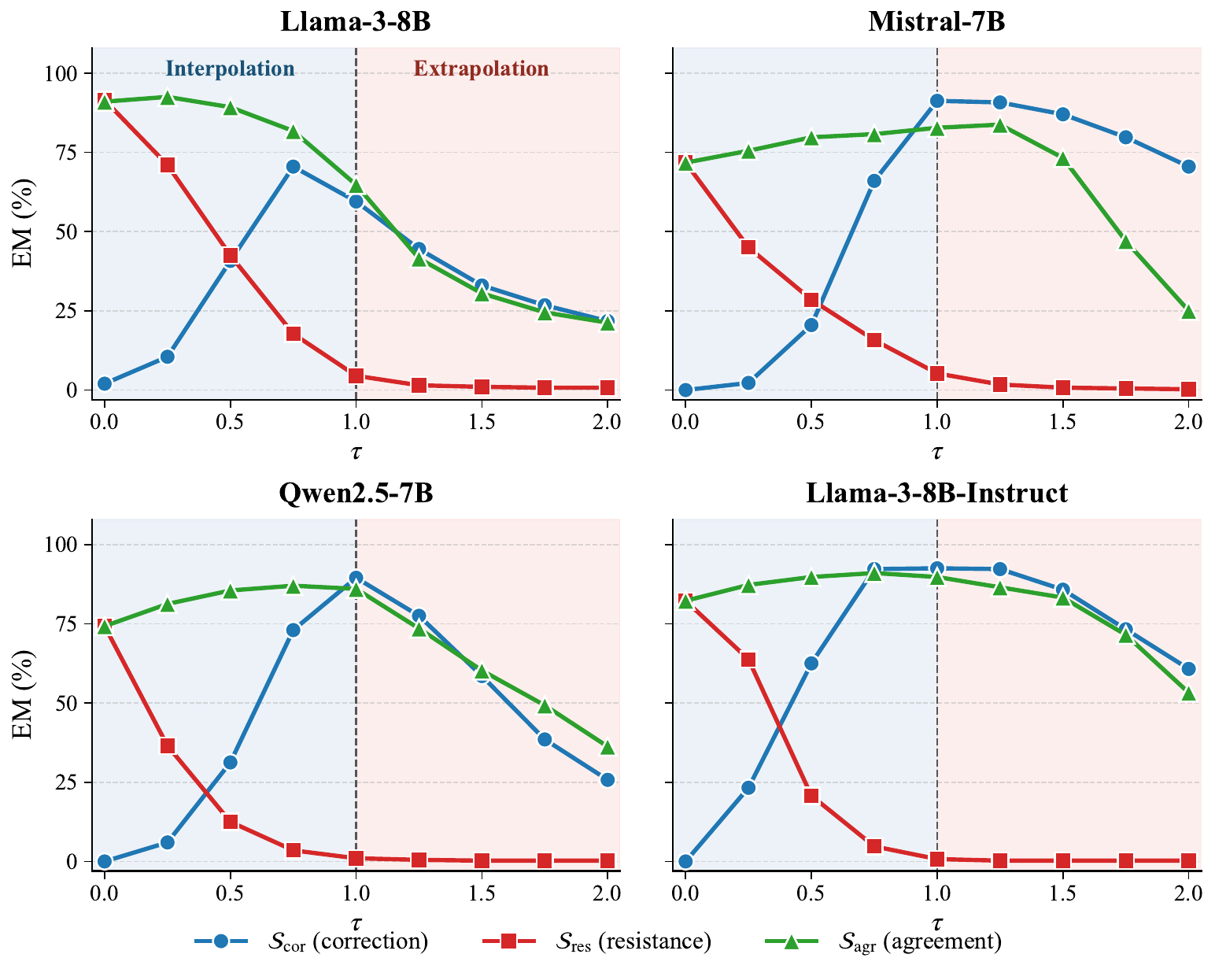}
  \caption{Blue/red shading separates the interpolation ($\tau\in[0,1]$) and extrapolation ($\tau>1$) regimes. $\mathcal{S}_{\mathrm{res}}$ decays monotonically with $\tau$; $\mathcal{S}_{\mathrm{cor}}$ peaks near $\tau\approx 1$ and collapses under extrapolation, with model-dependent severity (sharpest for Llama-3-8B, mildest for Llama-3-8B-Instruct).}
  \label{fig:threshold_change_full}
\end{figure*}

\begin{figure*}[t]
\centering
  \includegraphics[width=\linewidth]{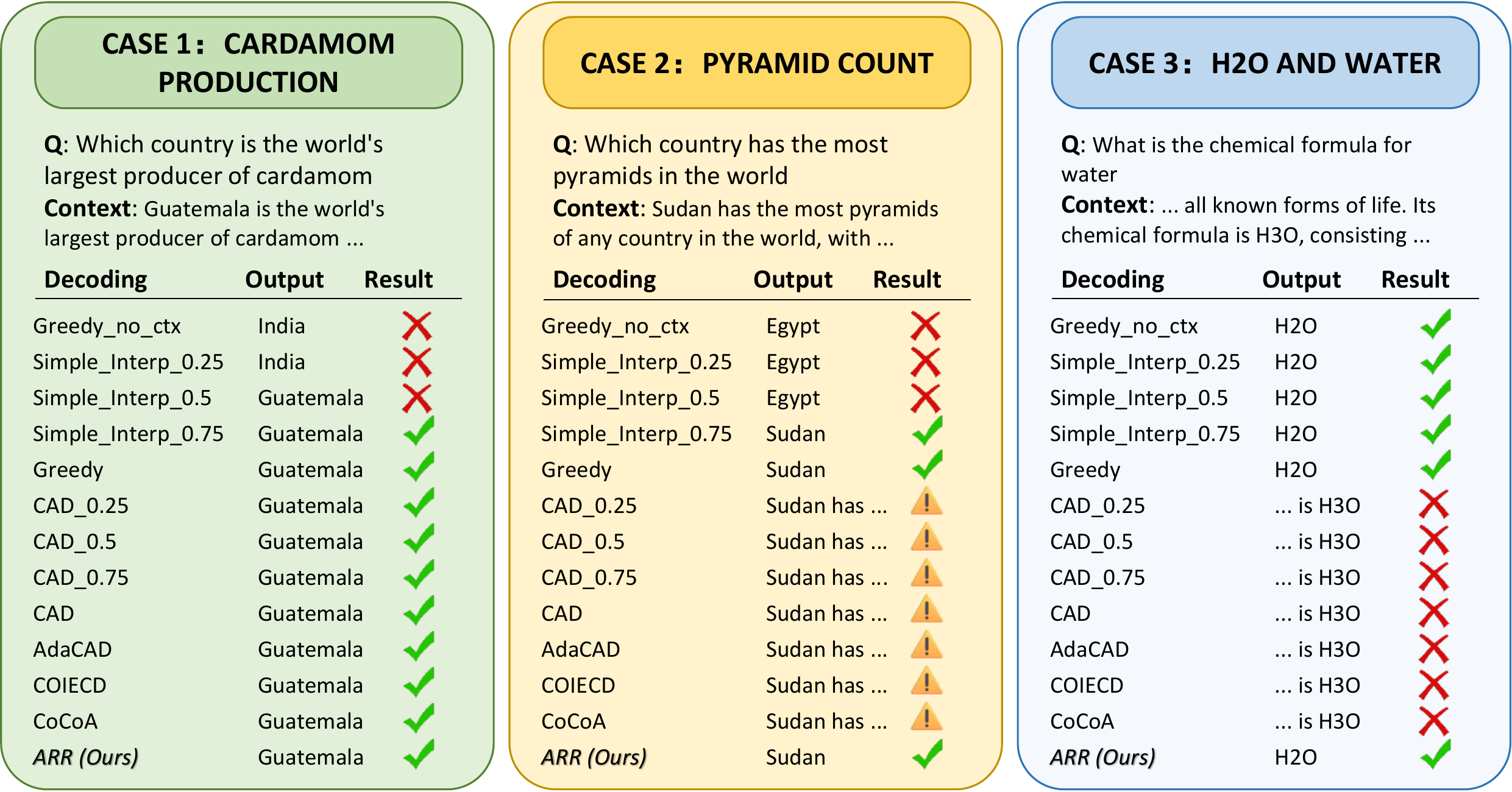}
  \caption{Detailed cases on Llama-3-8B, illustrating how varying $\tau$ affects decoding outputs and exposes the structural failure of extrapolation methods.}
  \label{fig:detailed_cases}
\end{figure*}

\begin{figure*}[t]
\centering
  \includegraphics[width=\linewidth]{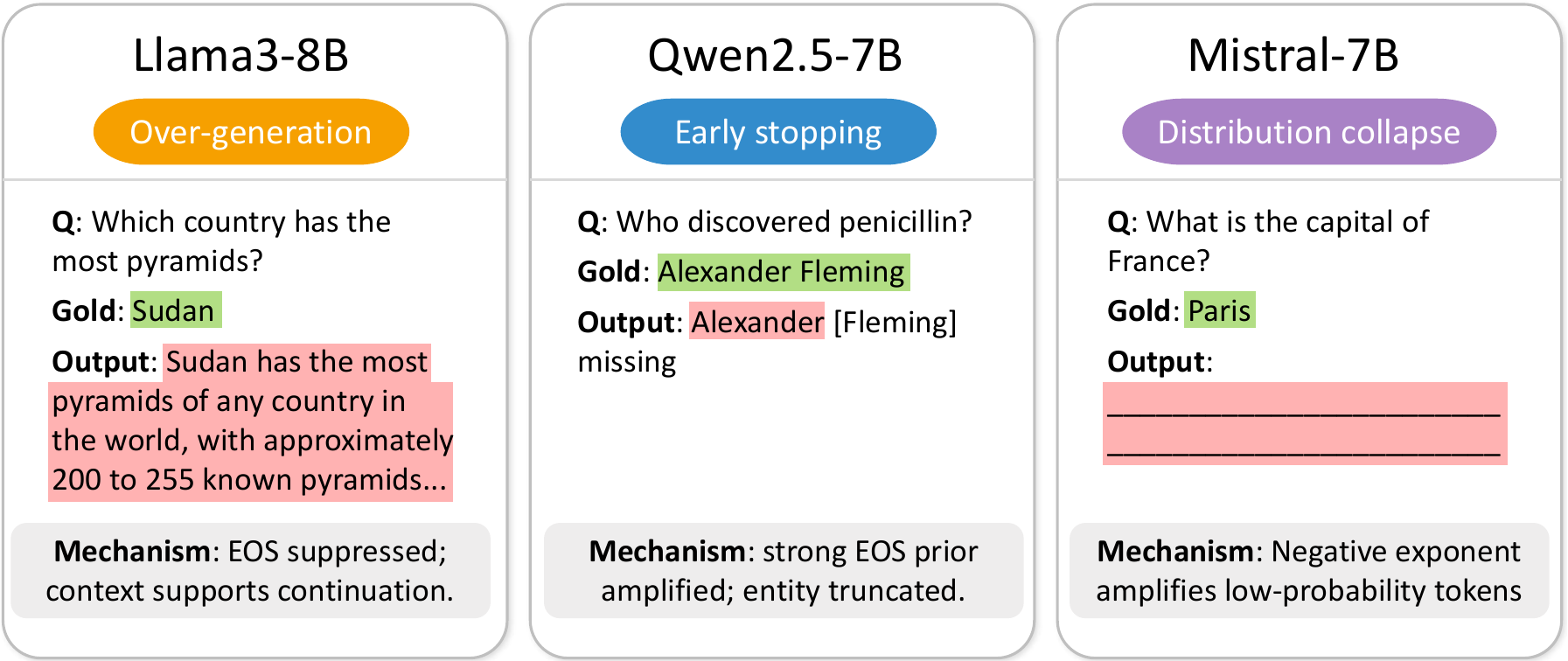}
  \caption{The detailed failure cases in TriState-Bench on Llama-3-8B, Qwen2.5-7B, and Mistral-7B.}
  \label{fig:detailed_failure_cases}
\end{figure*}

\section{Prompts}
\label{app:prompts}
We use the QA prompt template shown in Table~\ref{tab:prompt_template}.

\begin{table}[t]
\centering
\small
\setlength{\tabcolsep}{6pt}
\renewcommand{\arraystretch}{1.25}
\begin{tabular}{@{}p{0.46\linewidth} p{0.46\linewidth}@{}}

\toprule
\multicolumn{2}{c}{\textbf{Question Answering Prompt Template}} \\
\midrule
\multicolumn{1}{c}{\textit{With Context}} & \multicolumn{1}{c}{\textit{Without Context}} \\
\midrule
Using only the references listed below, answer the following question. \newline\newline
\textbf{Context}: \{context\} \newline
\textbf{Question}: \{question\}? \newline
\textbf{Answer}:
&
Answer the following question. \newline\newline
\textbf{Question}: \{question\}? \newline
\textbf{Answer}: \\
\bottomrule
\end{tabular}
\caption{Prompt templates}
\label{tab:prompt_template}
\end{table}

\end{CJK}
\end{document}